\documentclass[journal]{IEEEtran}
\usepackage[pdftex]{graphicx}
\graphicspath{{Figures}}
\usepackage{booktabs}
\usepackage{subfigure}
\usepackage{multirow}
\usepackage{amsthm}
\usepackage{amsmath}
\usepackage{amssymb}
\usepackage{mathtools}
\usepackage[ruled]{algorithm2e}
\usepackage{makecell}
\newtheorem{theorem}{Theorem}
\newtheorem{lemma}{Lemma}
\newtheorem{definition}{Definition}
\usepackage{color}
\usepackage{bm}
\newcommand{\chen}[1]{\textcolor{black}{#1}}

\begin{document}
\title{Motif Graph Neural Network}
\author{Xuexin Chen,  Ruichu Cai$^\star$, Yuan Fang$^\star$, Min Wu$^\star$, Zijian Li, Zhifeng Hao

\thanks{This research was supported in part by National Key R$\&$D Program of China (2021ZD0111501),  National Science Fund for Excellent Young Scholars (62122022), Natural Science Foundation of China (61876043, 61976052), the major key project of PCL (PCL2021A12), Guangdong Provincial Science and Technology Innovation Strategy Fund (2019B121203012).}

\IEEEcompsocitemizethanks{
		\IEEEcompsocthanksitem Xuexin Chen is with the School of Computer Science, Guangdong University of Technology, Guangzhou 510006, China.
		E-mail: im.chenxuexin@gmail.com
		\IEEEcompsocthanksitem Ruichu Cai is with the School of Computer Science, Guangdong University of Technology, Guangdong Provincial Key Laboratory of Public Finance and Taxation with Big Data Application, Guangzhou, China and also with Peng Cheng Laboratory, Shenzhen 518066, China. 
		E-mail: cairuichu@gmail.com
		\IEEEcompsocthanksitem Yuan Fang is with the School of Computing and Information Systems, Singapore Management University, 178902, Singapore.
		E-mail: yfang@smu.edu.sg
		\IEEEcompsocthanksitem Min Wu is with the Institute for Infocomm Research (I$^{2}$R), A*STAR, 138632, Singapore.
        E-mail: wumin@i2r.a-star.edu.sg
		\IEEEcompsocthanksitem Zijian Li is with the School of Computer Science, Guangdong University of Technology, Guangzhou 510006, China.
		E-mail: leizigin@gmail.com
		\IEEEcompsocthanksitem Zhifeng Hao is with the College of Science, Shantou University, Shantou 515063, China. 
		Email: haozhifeng@stu.edu.cn
}
}




\markboth{Journal of \LaTeX\ Class Files,~Vol.~14, No.~8, August~2015}%
{Shell \MakeLowercase{\textit{et al.}}: Bare Demo of IEEEtran.cls for IEEE Journals}

\maketitle

\begin{abstract}
Graphs can model complicated interactions between entities, which naturally emerge in many important applications. These applications can often be cast into standard graph learning tasks, in which a  crucial step is to learn low-dimensional graph representations. Graph neural networks (GNNs)  are currently the most popular model in graph embedding approaches. However, standard GNNs in the neighborhood aggregation paradigm suffer from limited discriminative power in distinguishing \emph{high-order} graph structures as opposed to \emph{low-order} structures. To capture high-order structures, researchers have resorted to motifs and developed motif-based GNNs. However, existing motif-based GNNs still often suffer from less discriminative power on high-order structures. To overcome the above limitations, we propose Motif Graph Neural Network (MGNN), a novel framework to better capture high-order structures, hinging on our proposed motif redundancy minimization operator and injective motif combination. First, MGNN produces a set of node representations w.r.t. each motif. The next phase is our proposed redundancy minimization among motifs which compares the motifs with each other and distills the features unique to each motif. Finally, MGNN performs the updating of node representations by combining multiple representations from different motifs. In particular,  to enhance the discriminative power, MGNN utilizes an injective function to combine the representations w.r.t. different motifs. We further show that our proposed architecture increases the expressive power of GNNs with a theoretical analysis. We demonstrate that MGNN outperforms state-of-the-art methods on seven public benchmarks on both node classification and graph classification tasks.
\end{abstract}

\begin{IEEEkeywords}
Graph Neural Network, Motif, High-order Structure, Graph Representation
\end{IEEEkeywords}

\section{Introduction}\label{sec:intro}
Graphs are capable of modeling complex interactions between entities, which naturally emerge in many real-world scenarios.
Social networks, protein-protein interaction networks, and knowledge graphs are just a few examples, with many important applications in areas like social recommendation \cite{DBLP:conf/www/Fan0LHZTY19}, drug discovery \cite{sun2020graph}, fraud detection \cite{xu2021towards}, and particle physics \cite{shlomi2020graph}. These applications can often be cast into standard graph learning tasks such as node classification, link prediction, and graph classification, in which a crucial step is to learn low-dimensional graph representations.

\begin{figure}[t]
	\centering
	\includegraphics[width=\columnwidth]{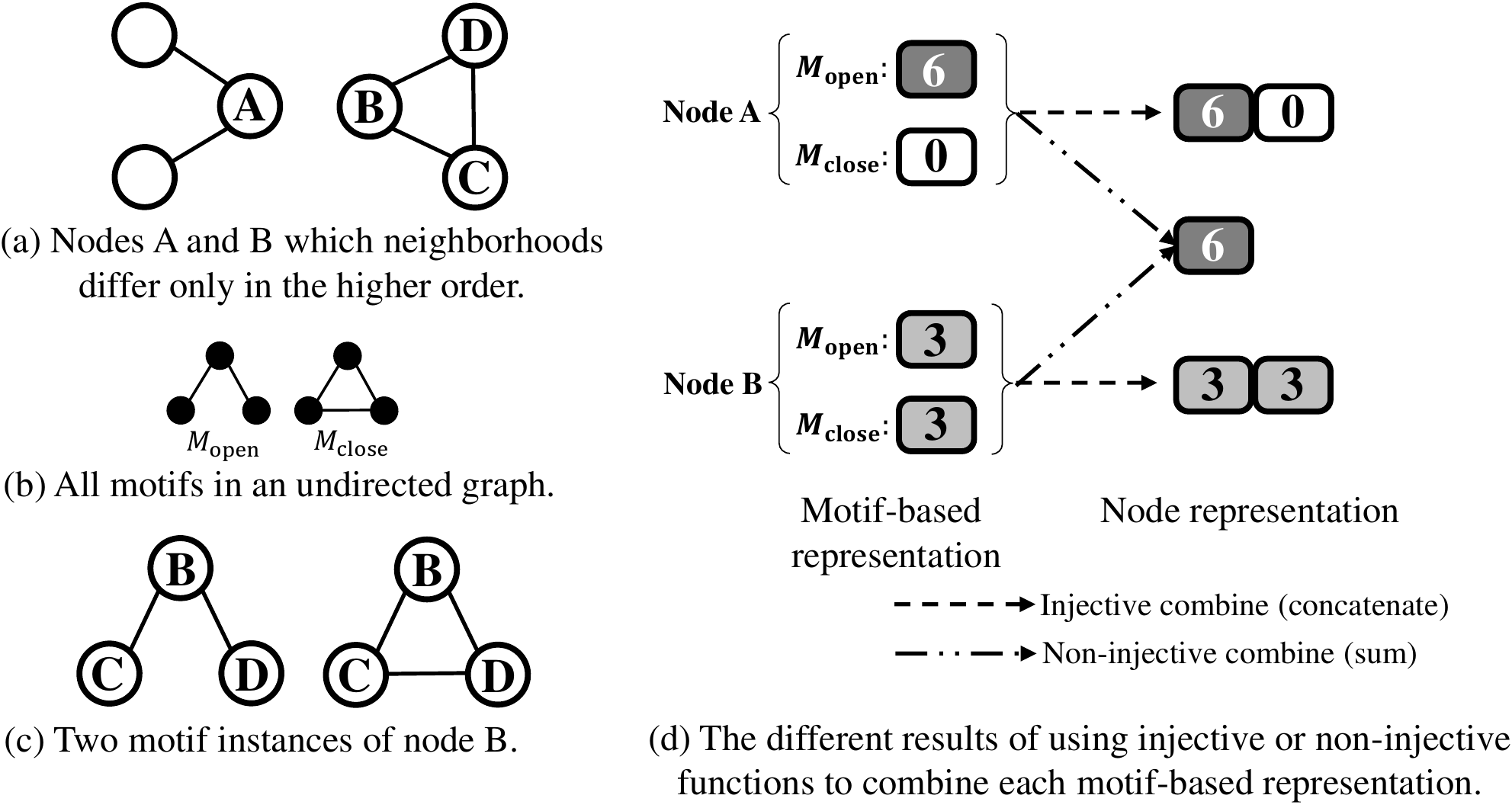} 
	\caption{Toy example of the discriminative power of GNNs on two nodes A and B with non-isomorphic neighborhoods.
	}
	\label{fig:example}
\end{figure}

Graph embedding approaches can be broadly categorized into graph neural networks (GNNs) \cite{DBLP:conf/iclr/KipfW17, DBLP:journals/corr/abs-1710-10903, hamilton2017inductive}, matrix factorization \cite{DBLP:conf/kdd/OuCPZ016, DBLP:conf/cikm/CaoLX15}  and skip-gram models \cite{DBLP:conf/kdd/GroverL16,DBLP:conf/kdd/PerozziAS14}. Among these, GNNs are currently the most popular model, largely owing to their ability of integrating both structure and content information through a message passing mechanism.
To be more specific, in the standard GNN architecture, the representation vector of a node is computed by aggregating and updating messages (i.e., features or representation vectors) from the node’s neighbors. The aggregation can be performed recursively by stacking multiple layers, to capture long-range node dependencies. 

However, standard GNNs in the neighborhood aggregation paradigm suffer from limited discriminative power in distinguishing \emph{high-order} graph structures consisting of the connections between neighbors of a node, as opposed to \emph{low-order} structures consisting of the connections between the node and its neighbors. For example, standard GNNs cannot distinguish between nodes A and B with non-isomorphic neighborhoods in Fig.~\ref{fig:example}(a), as their neighborhoods differ only in the higher-order. To capture high-order structures, researchers have resorted to motifs \cite{milo2002network, benson2016higher} and developed motif-based GNNs \cite{DBLP:conf/cikm/LeeRKKKR19,DBLP:conf/cikm/ZhaoZSL19,DBLP:conf/dsw/MontiOB18,DBLP:conf/icdm/SankarWKS20}. These approaches usually employ a motif-based adjacency matrix for each motif, which is constructed from the number of times two nodes are connected via an instance of the motif. Such motif-based adjacency matrices can better grasp the high-order structures. For example, given the open and closed motifs illustrated in Fig.~\ref{fig:example}(b), nodes A and B in Fig.~\ref{fig:example}(a) can be naturally distinguished by motif-based GNNs since node A is only associated with an open motif, whereas node B is associated with both open and close motifs.

However, existing motif-based GNNs often suffer from two problems. First, they overlook the \emph{redundancy} among motifs, which is defined as common edges shared by different motif instances. For example, in Fig.~\ref{fig:example}(c), the two motif instances of node B share two edges. When the redundancy is high enough, different motifs may become similar and lack specificity. Second, they often combine multiple motifs in a \emph{non-injective} manner, potentially resulting in less discriminative power on high-order structures. That is, a non-injective function, such as sum or mean, is used to combine different motifs, as shown in Fig.~\ref{fig:example}(d). In our example, node A has only an open motif with a feature valued 6, and node B has both open and close motifs each with a feature valued 3. However, when the motifs are combined by summing up their features, both nodes A and B would obtain the same feature representation of 6 and thus cannot be distinguished. Thus, the resulting node representations may also converge and further decrease the discriminative power.

To overcome the above limitations, we propose Motif Graph Neural Network (MGNN), a new class of GNN capable of distinguishing high-order structures with provably better discriminative power. 
From a model perspective, our MGNN follows the message passing mechanism, and its procedure is broken down into the following four phases. The first phase is motif instance counting. To form a motif-based adjacency matrix, we count the number of times two nodes co-occur in an instance of each motif. To capture comprehensive high-order graph structures, MGNN employs the motif-based adjacency matrices for all the possible motifs of size three, as opposed to some previous work with only one \cite{zhao2018ranking} or some motifs \cite{DBLP:conf/icdm/SankarWKS20}. The second phase is message aggregation. MGNN, like other motif-based  GNNs, 
aggregates node features (i.e., messages) on each motif-based adjacency matrix to produce different representations of the motifs. The first two phases are largely based on previous studies, except that we have employed all the motifs of size three to thoroughly capture high-order structures in an efficient manner. The third phase is the redundancy minimization among motifs. We address the challenge of redundancy among motifs by a proposed redundancy minimization operator, which compares the motifs with each other in terms of their representations, to distill the features specific to each motif. The fourth phase is the updating of node representations by combining multiple motifs. To improve the limited discriminative power of non-injective combinations, MGNN utilizes an injective function to combine motifs and update node representations. For example, to distinguish nodes A and B in Fig.~\ref{fig:example}(d), MGNN uses the injective concatenation to combine motif-based representations, so that the representation of node A is (6,0) and that of B is (3, 3), which can be differentiated apart. From a theoretical perspective, we show that MGNN is provably more expressive than standard GNN, and standard GNN is in fact a special case of MGNN.

We summarize our key contributions in the following.
\begin{itemize}
	\item We propose Motif Graph Neural Network (MGNN), a novel framework to better capture high-order structures, hinging on the motif redundancy minimization operator and injective motif combination. 
	\item We further show that our proposed architecture increases the expressive power of GNNs with a theoretical analysis.
	\item We demonstrate that MGNN outperforms existing standard or motif-based GNNs on seven public benchmarks on both node classification and graph classification tasks.
\end{itemize}

\section{Related Work}

Standard Graph Neural Networks follow the message passing paradigm to leverage node dependence and learn node representations. 
Different GNN models resort to different aggregation functions to aggregate the messages (i.e., features) for each node from its neighbors, and update its representation \cite{DBLP:conf/iclr/KipfW17, DBLP:journals/corr/abs-1710-10903, hamilton2017inductive, DBLP:conf/icml/GilmerSRVD17, DBLP:journals/corr/abs-1806-01261}\chen{, \cite{gan2022multigraph, he2022optimizing, gong2022self, guo2021bilinear}}. For example, Graph Convolutional Networks \cite{DBLP:conf/iclr/KipfW17} use mean aggregation to pool neighborhood information. 
Graph Attention Networks \cite{DBLP:journals/corr/abs-1710-10903} aggregate neighborhood information with trainable attention weights.
GraphSAGE \cite{hamilton2017inductive} uses mean, max or other learnable pooling function. 
Moreover, during aggregation, Message Passing Neural Networks \cite{DBLP:conf/icml/GilmerSRVD17} also incorporate edge information, while Graph Networks \cite{DBLP:journals/corr/abs-1806-01261} \chen{and multi graph fusion-based dynamic GCN~\cite{ gan2022multigraph}} further consider global graph information. \chen{In \cite{he2022optimizing} and \cite{gong2022self}, GNNs are developed to use an aggregation strategy based on Hilbert-Schmidt independence criterion and self-paced label augmentation strategy, respectively.} Some graph-level downstream tasks, such as graph classification, further employ a readout function to aggregate individual node representations into a whole-graph representation. The readout can be a simple permutation invariant function such as average and summation, while more sophisticated graph pooling methods have also been proposed,
including global pooling \cite{DBLP:journals/corr/VinyalsBK15,DBLP:journals/corr/LiTBZ15,zhang2018end},
hierarchical pooling \cite{DBLP:journals/pami/DhillonGK07,DBLP:conf/aaai/RanjanST20,diehl2019towards,DBLP:conf/icml/GaoJ19,DBLP:conf/nips/YingY0RHL18}. \chen{Besides graph-level downstream tasks, GNNs can also be used for the visual question answering task, etc~\cite{guo2021bilinear}.}
However, all these models are limited to only capturing low-order graph structures around every node.

However, standard GNNs are at most as powerful as the $1$-dimensional Weisfeiler-Leman ($1$-WL) graph isomorphism test \cite{DBLP:conf/iclr/XuHLJ19}, which implies that they cannot distinguish nodes with isomorphic low-order graph structures but different high-order structures. In other words, standard GNNs will always associate such nodes with the same representation. To improve the discriminative power of GNNs, it is a common practice to leverage high-order graph structures such as motifs \cite{milo2002network, benson2016higher}. 
In particular, motif-based GNN models use one \cite{zhao2018ranking, DBLP:conf/cikm/ZhaoZSL19}\chen{, \cite{DBLP:journals/corr/abs-2205-00867}} or more \cite{DBLP:conf/dsw/MontiOB18,DBLP:conf/icdm/SankarWKS20,DBLP:conf/bigdataconf/DareddyDY19,DBLP:journals/corr/abs-1711-05697,DBLP:conf/cikm/LeeRKKKR19}\chen{, \cite{wang2022graph}, \cite{yang2022graph}} motif-based adjacency matrices to perform message passing. When multiple motif-based adjacency matrices are used, the combination function w.r.t.~multiple motifs include summation \cite{DBLP:conf/dsw/MontiOB18}\chen{, \cite{wang2022graph}}, averaging \cite{DBLP:conf/icdm/SankarWKS20}, neighborhood aggregation \cite{DBLP:conf/bigdataconf/DareddyDY19}, fusion by a fully connected layer \cite{DBLP:journals/corr/abs-1711-05697}, selection by reinforcement learning \cite{DBLP:conf/cikm/LeeRKKKR19}, \chen{combination by a variant of recurrent neural network \cite{yang2022graph}}, and so on. However, all these previous combination functions are not injective to sufficiently differentiate higher-order structures.
\chen{Note that although the model in \cite{yang2022graph} does not employ an injective function, it still effectively captures the high-order structure of the nodes, through a strategy of encoding neighbor's features sequentially and a variant of the recurrent neural network to learn the node representations.}
Besides, all these models do not take into account the redundancy among motif instances. 

In another line, several studies \cite{maron2019provably, DBLP:conf/aaai/0001RFHLRG19, morris2020weisfeiler} attempt to extend the discriminative power of GNNs from 1-WL to $k$-WL, given that the higher the dimension of WL, the stronger the discriminative power. 
Like the standard GNNs, there is also a message propagation mechanism in these $k$-WL approaches 
where the difference is that the message is not propagated between nodes but between $k$-tuples (or a subgraph with $k$ nodes). 
Since their message propagation is not between nodes, they \cite{maron2019provably,DBLP:conf/aaai/0001RFHLRG19,morris2020weisfeiler} have the following shortcomings compared with our MGNN. 
First, 
they cannot generate node embeddings, but MGNN can, which limits their application to node-level tasks such as node classification. 
Second, their time complexity is higher than that of MGNN. The time complexity of MGNN is $\mathcal{O}(|\mathcal{V}|^2)$ (see Section~\ref{sec:model_train}), while their time complexity is $\mathcal{O}(|\mathcal{V}|^3)$ \cite{maron2019provably} or even $\mathcal{O}(|\mathcal{V}|^4)$ in the worst case  \cite{DBLP:conf/aaai/0001RFHLRG19,morris2020weisfeiler}. 
Third, \cite{DBLP:conf/aaai/0001RFHLRG19,morris2020weisfeiler} have a space complexity $\mathcal{O}(|\mathcal{V}|^3)$, which is also higher than $\mathcal{O}(|\mathcal{V}|^2)$ needed by MGNN. 


There are also several approaches employing high-order structures, in which each node receives messages from its multi-hop neighbors, such as MixHop~\cite{abu2019mixhop}, GDC~\cite{klicpera2019diffusion}, CADNet~\cite{lim2021class}, \chen{PathGCN~\cite{flam2021neural}, SE-aggregation~\cite{zhang2021learning} and MBRec~\cite{xia2022multi}}. However, like standard GNNs, they are \chen{typically} at most as powerful as the $1$-WL test in distinguishing graph structures~\chen{\cite{zhang2021learning}}.

\section{Preliminaries}

In this section, we introduce major notations and definitions of related concepts.

\subsection{Notations and Problem Formulation}
A graph is denoted by $G=(\mathcal{V}, \mathcal{E})$, with the set of nodes $\mathcal{V}$ and the set of edges $\mathcal{E}$. Let $\mathbf{A} \in \mathbb{R}^{|\mathcal{V}| \times |\mathcal{V}|}$ be the adjacency matrix of $G$, and $\mathbf{X} \in \mathbb{R}^{|\mathcal{V}| \times d_0}$ be the node feature matrix of G, where $i$-th row means the features of node $i$ denoted by $\mathbf{x}_i$. We use $(\mathbf{T})_{ij}$ to represent the element in the $i$-th row and the $j$-th column of a matrix $\mathbf{T}$, and $(\mathbf{T})_{i*}$ to represent all of the $i$-th row's elements.

In this paper, we investigate the problem of graph representation learning, which aims to embed nodes into a low-dimensional space. The node embeddings can be used for downstream tasks such as node classification, potentially in an end-to-end fashion.
Formally, a node embedding model is denoted by a function $\psi: \mathcal{V} \to \mathcal{H}$ that maps the nodes in $\mathcal{V}$ to $d$-dimensional vectors in $\mathcal{H} = \{\mathbf{h}_i \in \mathbb{R}^d|1\le i\le |\mathcal{V}| \}$, where $i$ denotes the index of the nodes.

\begin{figure}[t!]
	\centering
	\includegraphics[width=0.5\columnwidth]{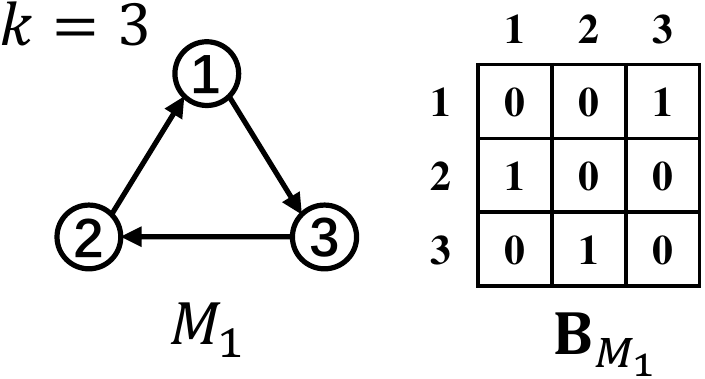} 
	\caption{An example of a 3-node ($k=3$) network motif, along with its adjacency matrix $\mathbf{B}_{M_1}$.
	}
	\label{fig:motif_def}
\end{figure}

\begin{figure}[t!]
	\centering
	\includegraphics[width=\columnwidth]{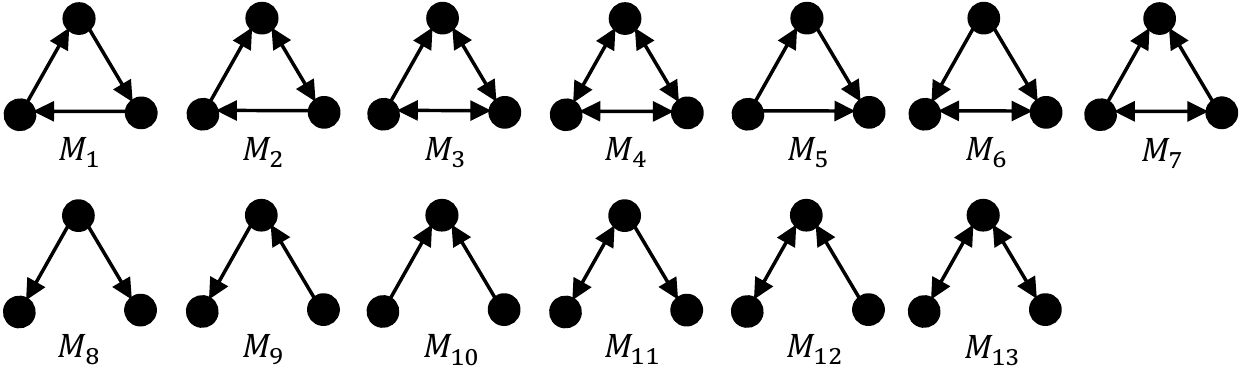} 
	\caption{All 3-node motifs in a directed and unweighted graph.}
	\label{fig:motif}
\end{figure}

\subsection{Motif and Motif-based Adjacency Matrix}
We work with directed motifs because they allow us to describe more complex structures. Specifically, we first introduce the definition of motif~\cite{milo2002network, zhao2018ranking, benson2016higher} as follows.

\begin{definition}\label{def:motif} (Network motif). 
A motif $M$ is a connected graph of $n$ nodes ($n>1$), with a $n \times n$ adjacency matrix $\mathbf{B}_M$ containing binary elements $\{0,1\}$.
\end{definition}

An example of 3-node motif is given in Fig.~\ref{fig:motif_def}.  
In particular, a motif with three or more nodes (i.e., $n \ge 3$) can capture high-order graph structures. Among them, w.r.t.~a given node, the high-order structure captured by its motifs with $n > 3$ nodes (i.e., not only its edges incident to its neighboring nodes but also the edges between its neighboring nodes), can be similarly captured by multiple $3$-node motifs. Thus, the given node's 3-node motifs have sufficient capacity for structures.
As shown in Fig.~\ref{fig:motif}, we enumerate a total of thirteen 3-node motifs. Therefore, we only utilize motifs with $n=3$ nodes in this work.

Given the above motif definition, we can further define the set of motif instances as follows.

\begin{definition}\label{def:motif_instance} (Motif instance).
Consider an edge set $\mathcal{E}'$ and the subgraph $G[\mathcal{E}']$ induced from $\mathcal{E}'$ in $G$. If $G[\mathcal{E}']$ and a motif $M_k$ are isomorphic \cite{babai2018group}, written as $M_k \simeq G[\mathcal{E}']$, then
\begin{equation*}
    m(\mathcal{E}') = \{(\mathbf{x}_{u}, \mathbf{x}_{v}) \big| (u, v) \in \mathcal{E}' \}
\end{equation*}
is an \emph{instance} of the motif $M_k$,
where $u, v$ are two adjacent nodes that form an edge in $\mathcal{E}'$, and $\mathbf{x}_{u}$ means the $u$-th row of $\mathbf{X}$ (i.e., the feature vector of node $u$).
\end{definition}
For example, a motif instance of $M_1$ in Fig.~\ref{fig:motif_def} is $\{(\mathbf{x}_1, \mathbf{x}_3), (\mathbf{x}_2, \mathbf{x}_1),  (\mathbf{x}_3, \mathbf{x}_2) \}$.

\begin{definition}\label{def:motif_instance_set} (Motif instance set). On a graph $G=(\mathcal{V}, \mathcal{E})$, the set of instances of motif $M_k$, denoted as $\mathcal{M}_k$, is defined by 
\begin{equation*}
\mathcal{M}_k = \{ m(\mathcal{E}') | \mathcal{E}' \subseteq \mathcal{E}, |\mathcal{E}'|=r, M_k \simeq G[\mathcal{E}'] \},
\end{equation*}
where $\mathcal{E}' \subseteq \mathcal{E}, |\mathcal{E}'|=r$ denotes the set of all $r$-combinations of the edge set $\mathcal{E}$, and $|\mathcal{E}'|=r$ is the number of edges in the motif $M_k$. 
\end{definition}

Based on the motif instances, the definition of the motif-based adjacency matrix is given as follows.
\begin{definition}\label{def:motif_adj} (Motif-based adjacency matrix). 
Given a motif $M_k$ and its set of instances $\mathcal{M}_k$, the corresponding motif-based adjacency matrix $\mathbf{A}_k$ is defined by
\begin{equation}\label{equ:motif_adj}
(\mathbf{A}_k)_{ij} = \sum_{m \in \mathcal{M}_k} \mathbb{I}((\mathbf{x}_i, \mathbf{x}_j) \in m),
\end{equation}
where $\mathbb{I}(\cdot)$ is an indicator function, i.e., $\mathbb{I}(x)=1$ if the statement $x$ is true and $0$ otherwise.
\end{definition}
Intuitively, $(\mathbf{A}_k)_{ij}$ is the number of times two nodes $i$ and $j$ are connected via an instance of the motif $M_k$.



\begin{figure*}[t!]
	\centering
	\includegraphics[width=2.\columnwidth]{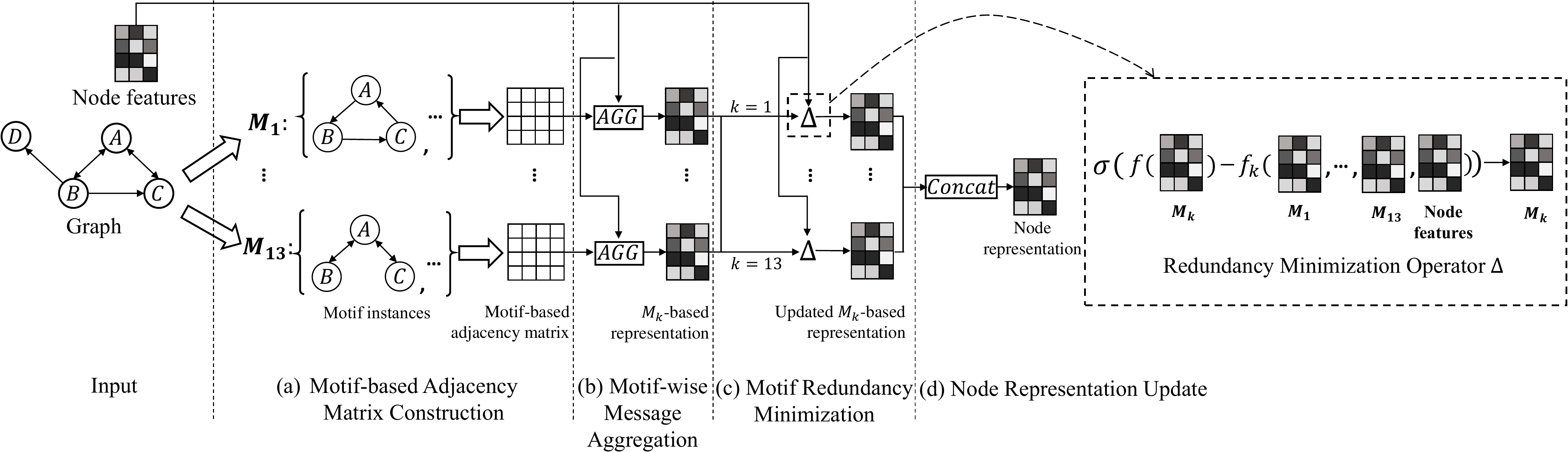} 
	\caption{Overview of a MGNN layer. The MGNN layer takes the graph and its node features matrix as inputs, throughout the procedure (a) (b) (c) (d), and finally outputs the node representation matrix that captures the high-order graph structure. The first two phases are the generation of $M_k$-based representation of nodes. The key idea of phase (c) is to compare the motifs with each other and distill the features specific to each motif. In phase (d), updated $M_k$-based representations are combined by an injective concatenation operation.
	}
	\label{fig:model}
\end{figure*}

\section{Proposed Approach}\label{sec:approach}
In this section, we introduce the proposed approach. We first present an overall framework of our approach, followed by its four phases in detail. Finally, we discuss the overall objective function for model training.

\subsection{Overall Framework}
We propose Motif Graph Neural Network (MGNN) that can model high-order structures with provably better discriminative power. Specifically, our MGNN follows a message passing mechanism, and its procedure is broken down into the following four phases. 

The first phase involves the construction of a motif-based adjacency matrix, as shown in Fig.~\ref{fig:model}(a). Given a motif, its motif-based adjacency matrix captures the number of times each pair of nodes are connected via an instance of the motif. Thus, we need an efficient counting algorithm for motif instances. In MGNN, we consider all 13 motifs of size three, namely $M_1,M_2,\ldots,M_{13}$ given by Fig.~\ref{fig:motif}, and correspondingly construct 13 motif-based adjacency matrices $\mathbf{A}_1,\mathbf{A}_2,\ldots,\mathbf{A}_{13}$. 
The second phase is message aggregation, as shown in Fig.~\ref{fig:model}(b). MGNN aggregates node features (i.e., messages) on each motif-based adjacency matrix to produce a set of node representations w.r.t.~each motif. The first two phases of motif instance counting \cite{zhao2018ranking} and message aggregation \cite{DBLP:conf/dsw/MontiOB18, DBLP:journals/corr/abs-1711-05697} are largely similar to previous works, except that we have employed all the motifs of size three to comprehensively capture high-order structures in an efficient manner. In the second phase, we follow previous work completely. The third phase is the redundancy minimization among motifs, as shown in Fig.~\ref{fig:model}(c). We propose a redundancy minimization operator, which compares the motifs with each other and distills the features unique to each motif.
The fourth phase performs the updating of node representations  by  combining multiple representations from different motifs, as shown in Fig.~\ref{fig:model}(d).
In particular, to enhance the discriminative power,  MGNN utilizes an injective function to combine the representations w.r.t.~different motifs.


\subsection{Motif-based Adjacency Matrix Construction}\label{sec:adj}
\begin{figure}[t]
	\centering
	\includegraphics[width=\columnwidth]{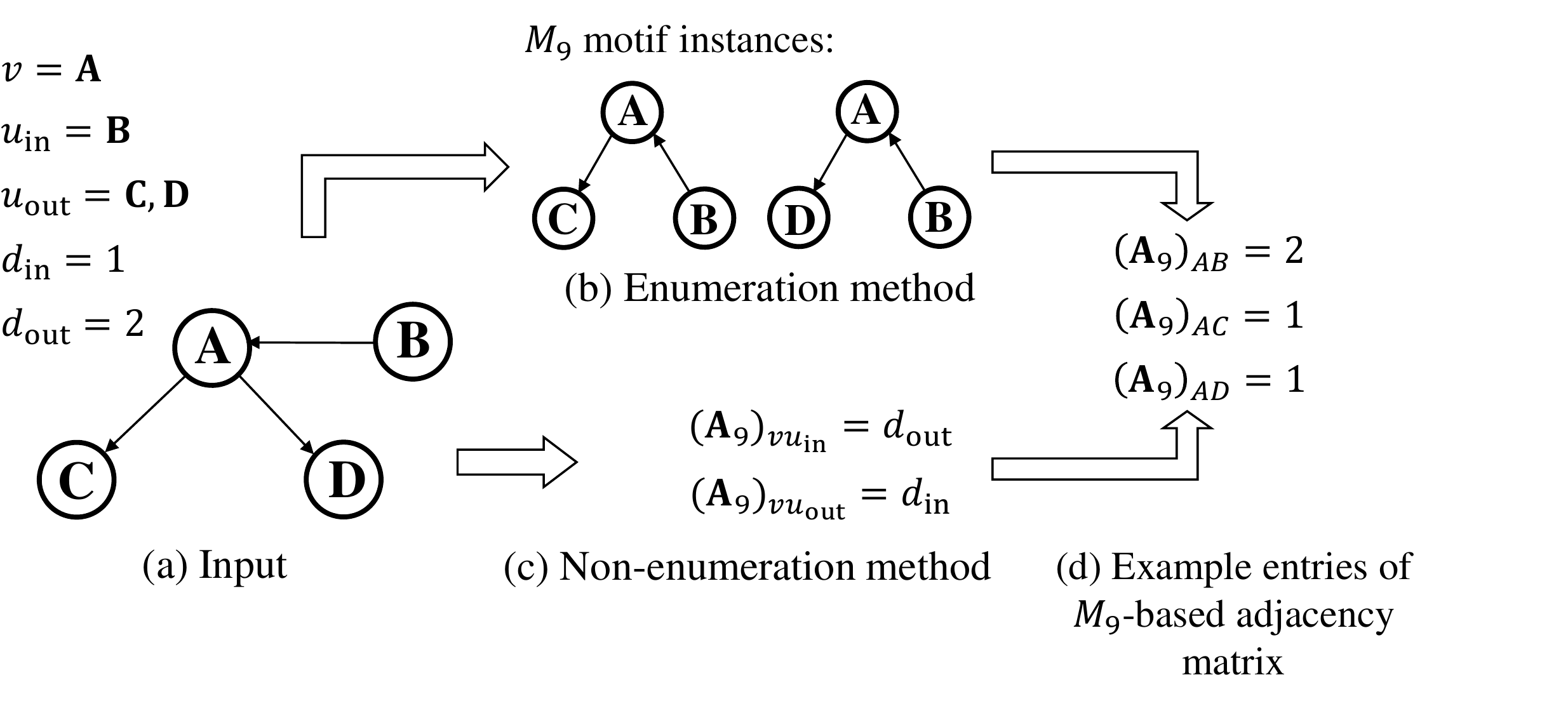} 
	\caption{Overview for constructing $M_9$-based adjacency matrix by enumeration and non-enumeration method.}
	\label{fig:motif_adj_construct}
\end{figure}

The key step to  constructing a motif-based adjacency matrix is to efficiently count the number of motif instances.
Depending on if the motif is open ($M_8$--$M_{13}$) or closed ($M_1$--$M_7$), different counting algorithms will apply.

For open motifs ($M_8$--$M_{13}$), existing methods \cite{zhang2019local,ribeiro2021survey} 
are often implemented by enumerating motif instances. For example, given the graph in Fig.~\ref{fig:motif_adj_construct}(a), to construct the $M_9$-based adjacency matrix, a traditional technique is to enumerate the instances of $M_9$ as shown in Fig.~\ref{fig:motif_adj_construct}(b). However, such enumeration suffers from high computational complexity, with a worst-case complexity of $O(|\mathcal{V}|^3)$ in both space and time.
To reduce the complexity, we propose an adjacency matrix construction method for open motifs without enumerating any motif instance, which has a time and space complexity of $O(|\mathcal{V}|^2)$ and $O(|\mathcal{V}|)$, respectively. Consider a node $v$. Let $u_\text{in}$, $u_\text{out}$, and $u_\text{bi}$ denote an incoming, outgoing, and bi-directional neighbor of node $v$, respectively. Correspondingly, let $d_\text{in}$, $d_\text{out}$ and $d_\text{bi}$  denote the number of each type of neighbor of $v$, respectively, as illustrated by the examples in Fig.~\ref{fig:motif_adj_construct}(a). As shown in Fig.~\ref{fig:motif}, the center node of each open motif has at most two types of neighbors; for example, $M_9$ has $u_\text{out}$ and $u_\text{in}$, and $M_{13}$ has only $u_\text{bi}$. Our key observation is that $(\mathbf{A}_k)_{vu}$, the number of times two nodes ($v$ and $u$) are connected via an instance of an open motif $M_k$, can be computed as follows. On one hand, when the motif has two types of neighbors, 
$(\mathbf{A}_k)_{vu}$ will be equal to the number of the other type of neighbors, e.g., $(\mathbf{A}_9)_{vu_\text{in}} = d_\text{out}$, $(\mathbf{A}_9)_{vu_\text{out}} = d_\text{in}$, $(\mathbf{A}_{11})_{vu_\text{out}} = d_\text{bi}$,  $(\mathbf{A}_{12})_{vu_\text{in}}=d_\text{bi}$,
$(\mathbf{A}_{11})_{vu_\text{bi}} =  d_\text{out} - 1$, $(\mathbf{A}_{12})_{vu_\text{bi}} =  d_\text{in} - 1$. On the other hand, when the motif has only one type of neighbors, $(\mathbf{A}_k)_{vu}$ will be equal to the number of neighbors in the motif, e.g., $(\mathbf{A}_8)_{vu_\text{out}} = d_\text{out} - 1$, $(\mathbf{A}_{10})_{v u_\text{in}} = d_\text{in} - 1$, $(\mathbf{A}_{13})_{vu_\text{bi}} = d_\text{bi} - 1$. 
Still using Fig.~\ref{fig:motif_adj_construct} as an example, node B is an incoming neighbor ($u_{\text{in}}$) of node A, while C and D are the outgoing neighbors ($u_{\text{out}}$) of node A. Furthermore, for $(\mathbf{A}_9)_{AB}$, it satisfies $(\mathbf{A}_9)_{vu_\text{in}}=d_\text{out}$ (denoting node A as $v$) and for $(\mathbf{A}_9)_{AC}$ or $(\mathbf{A}_9)_{AD}$, it satisfies $(\mathbf{A}_9)_{vu_\text{out}} = d_\text{in}$. 


The motif-based adjacency matrix for a closed motif ($M_1$--$M_7$) can be constructed by an existing method \cite{zhao2018ranking} with a time and space complexity of $O(|\mathcal{V}|^3)$ and $O(|\mathcal{V}|^2)$, respectively, which counts $(\mathbf{A}_k)_{vu}$ through two matrix multiplication operations and the matrices used by this method can be stored in
the HDF5 format \cite{hdf1997hierarchical}. 

\subsection{Motif-wise Message Aggregation}
To produce the motif-wise node representations, on each motif-based adjacency matrix, node features (i.e, messages) can be incorporated into a multi-layer message aggregation mechanism, as shown in Fig.~\ref{fig:model}(b). 

Specifically, the motif $M_k$-based representation of node $v$ in the $l$-th layer is given by
\begin{equation}\label{equ:maf}
\mathbf{h}_{v, k}^{(l)} = \operatorname{AGG}\left( \left\{ \alpha^{(l)}_{k,vi}\cdot (\tilde{\mathbf{A}}_k)_{vi} (\mathbf{Z}^{(l)})_{i*}| i \in \mathcal{N}(v) \right\} \right),
\end{equation}
\begin{equation}\label{equ:gcn}
\mathbf{Z}^{(l)} = \tilde{\mathbf{A}}\mathbf{H}^{(l-1)}\mathbf{W}^{(l)},
\end{equation}
where $\mathbf{H}^{(l-1)} \in \mathbb{R}^{|\mathcal{V}| \times d_{l-1}}$ denote the node messages from the previous $(l-1)$-th layer and $\mathbf{H}^{(0)} = \mathbf{X}$, $\mathbf{W}^{(l)} \in \mathbb{R}^{d_{l-1} \times d_l}$ is the trainable weight matrix in the $l$-th layer. $\tilde{\mathbf{A}}$ is the normalized adjacency matrix given by $\tilde{\mathbf{A}} = \hat{\mathbf{A}} - \frac{\hat{\lambda}_{\max}}{2}\mathbf{I}$, where $\hat{\mathbf{A}} = \mathbf{D}^{-0.5}\mathbf{A}\mathbf{D}^{-0.5}$ and  $\mathbf{D}$ is a diagonal
matrix in which the diagonal elements are defined as $(\mathbf{D})_{ii} = \sum_{j=1}^{|\mathcal{V}|} (\mathbf{A})_{ij}$ and $\hat{\lambda}_{\max}$ refers to the largest eigenvalue of $\hat{\mathbf{A}}$. The above normalization technique aids in the centralization of the Laplacian's  eigenvalues and the reduction of the spectral radius bound~\cite{DBLP:conf/nips/WijesingheW19}. The motif $M_k$-based adjacency matrix $\mathbf{A}_k$ is normalized in the same way into $\tilde{\mathbf{A}}_k$. AGG, the function of $\mathbf{H}^{(l-1)}$ and $\tilde{\mathbf{A}}_k$ as Fig.~\ref{fig:model}(b) illustrated, is a message aggregate function, e.g., sum, mean or max. The coefficient $\alpha^{(l)}_{k,vi}$ is the attention weight that indicates the importance of node $i$’s messages to node $v$. $\alpha^{(l)}_{k,vi}$ can be assigned a constant value according to prior knowledge or computed by the attention mechanism \cite{DBLP:journals/corr/abs-1710-10903}. $\mathcal{N}(v)$ represents the set of neighboring nodes of $v$. 
Note that not all nodes will have 13 motifs, and MGNN can still accommodate such nodes. In particular, if node $v$ lacks a motif $M_k$, the entries $(A_k)_{vj}$ and $(A_k)_{jv}$ in Eq.~\eqref{equ:motif_adj} are all set to zeros, and subsequently, $\mathbf{h}^{(l)}_{v,k}$ in Eq.~\eqref{equ:maf} will also be a zero vector. 


Intuitively, in Eq.~\eqref{equ:maf}, before performing motif-wise aggregation for the motif $M_k$, we first stack a GCN layer \cite{DBLP:conf/iclr/KipfW17}, i.e., $\mathbf{Z} = \tilde{\mathbf{A}}\mathbf{H}^{(l-1)}\mathbf{W}^{(l)}$ in Eq.~\eqref{equ:gcn}, to update the overall node messages by aggregating from the previous layer. The GCN layer can also be replaced by other GNN layers.



\subsection{Motif Redundancy Minimization}
As different motifs often share certain substructures, their corresponding motif-wise representations may become similar and lack specificity.   
Inspired by the idea of redundancy minimization between features \cite{DBLP:journals/pami/PengLD05}, we propose a redundancy minimization operator at the motif level, denoted $\Delta$. The key idea of $\Delta$ is to compare the motifs with each other and adaptively distill the features specific to each motif. We formally define $\Delta$ as follows.
Given a node $v$, for simplicity, let $\mathbf{h}_k$ and $\mathbf{z}_v$ denote $\mathbf{h}^{(l)}_{v, k}$, $(\mathbf{Z}^{(l)})_{v*}$ respectively.
We call the motif- and GCN-based representations collectively as the intermediate representations of the node.
\begin{definition}\label{def:op} (Motif redundancy minimization operator). 
For any node $v$, given its intermediate representations $\mathbf{h}_1$, ..., $\mathbf{h}_{13}, \mathbf{z}_v$, let $\bar{\mathcal{H}}_k=\Big(\big\|_{i=1,i\ne k}^{13} \mathbf{h}_i\Big) \big\| \mathbf{z}_v$, where $\|$ is the concatenation operator. In other words, $\bar{\mathcal{H}}_k$ concatenates all the intermediate representations except that based on motif $M_k$. Then, for motif $M_k$, its redundancy minimized representation of the node $v$ is given by
\begin{equation}\label{equ:op}
\begin{aligned}
\tilde{\mathbf{h}}_k&=\Delta(k, \mathbf{h}_1, ..., \mathbf{h}_{13}, \mathbf{z}_v)\\
&= \sigma \Big(\beta_k \cdot \big( f(\mathbf{h}_k) - f_k(\bar{\mathcal{H}}_k) \big) \Big).
\end{aligned}
\end{equation}
$\tilde{\mathbf{h}}_k$ is the updated representation of $\mathbf{h}_k$ after redundancy minimization.
$f: \mathbb{R}^d \to \mathbb{R}^{d'}$ is a learnable projection function to map the intermediate motif-based representations to the same space as its redundant features.
And $f_k: \mathbb{R}^{13d} \to \mathbb{R}^{d'}$ is a learnable feature selection function, which selects the redundant features w.r.t.~motif $M_k$. 
$\beta_k$ is the similarity between $f(\mathbf{h}_k)$ and $f_k(\bar{\mathcal{H}}_k)$, which acts as a regularizer
to prevent extremely small or large differences between the two terms.  $\sigma$ is an activation function (e.g., ReLU). 
\end{definition}

Intuitively, in Eq.~\eqref{equ:op}, the motif redundancy minimization operator subtracts or removes redundant features w.r.t.~each motif from the intermediate representations of a given node. Apart from minimizing the redundancy, the operator also performs an adaptive selection of motifs in general. That is, for an unimportant motif $M_k$, 
this operator will make $\tilde{\mathbf{h}}_{k}$ in Eq.~\eqref{equ:op} close to a zero vector through functions $f$ and $f_k$. In particular, when $\tilde{\mathbf{h}}_{k}$ is a zero vector, it is equivalent to removing the instance of $M_k$ containing node $v$ in Eq.~\eqref{equ:motif_adj}. In Section~\ref{sec:case}, we will use a heatmap to demonstrate this adaptive selection mechanism, which improves the robustness of MGNN.

To realize the motif redundancy minimization operator, we need to instantiate $f$, $f_k$ and $\beta_k$ in Eq.~\eqref{equ:op}.
In particular, we use a fully connected layer to fit $f$ and $f_k$, namely,
\begin{equation}\label{equ:f1}
f(\mathbf{h}_k) = \mathbf{W}_f^{(l)} \mathbf{h}_{k} + \mathbf{b}_f^{(l)},
\end{equation}
\begin{equation}\label{equ:f2}
f_k(\bar{\mathcal{H}}_k) = \mathbf{W}^{(l)}_{f_k} \bar{\mathcal{H}}_k + \mathbf{b}_{f_k}^{(l)},
\end{equation}
where $\mathbf{W}_f^{(l)} \in \mathbb{R}^{d'_{l} \times d_{l}}$ is a trainable matrix in the $l$-th layer shared by all motifs, and $\mathbf{W}^{(l)}_{f_k} \in \mathbb{R}^{d'_{l} \times  (13 d_{l})}$ is a trainable matrix specific to motif $M_k$ in the $l$-th layer, and $\mathbf{b}_f^{(l)} \in \mathbb{R}^{d'_l}$ and $\mathbf{b}_{f_k}^{(l)} \in \mathbb{R}^{d'_l}$ are the corresponding bias vectors. 
Furthermore, to measure the similarity $\beta_k$, 
we use the inner product with a  non-linear  activation (e.g., sigmoid or tanh),
that is,
\begin{equation}\label{equ:att}
\beta_k= \sigma \big(f(\mathbf{h}_k)^\intercal f_k(\bar{\mathcal{H}}_k)\big).
\end{equation}
$\sigma$ in Eq.\eqref{equ:op} and Eq.\eqref{equ:att} might be the same or distinct.


Through the above instantiations, we can minimize the redundancy among motifs as shown in Fig.~\ref{fig:model}(c), for every node in every layer. That is,
\begin{equation}\label{equ:instance}
\tilde{\mathbf{h}}_{v, k}^{(l)} =\Delta\big(k, \mathbf{h}_{v, 1}^{(l)}, ..., \mathbf{h}_{v, 13}^{(l)}, (\mathbf{Z}^{(l)})_{v*}\big),
\end{equation}
where $\tilde{\mathbf{h}}_{v, k}^{(l)} \in \mathbb{R}^{d'_l}$ is the updated representation of node $v$ based on motif $M_k$ in the $l$-th layer.







\subsection{Node Representation Update via Injective Function}\label{sec:update}
As shown in Fig.~\ref{fig:model}(d), MGNN updates the node representation by combining their intermediate, motif-based representations. To  improve  the  discriminative power on high-order structures, MGNN utilizes an injective function to combine different motif-based representations of each node, to update the output node representations in each layer. Specifically, we use the injective vector concatenation function, and generate the output node representation in the $l$-th layer below.
\begin{equation}\label{equ:concat}
\mathbf{h}^{(l)}_v = \big\|_{k=1}^{13} \tilde{\mathbf{h}}_{v, k}^{(l)}, 
\end{equation}
where $\mathbf{h}^{(l)}_v \in \mathbb{R}^{13 d_l}$ is the output representation of node $v$ in the $l$-th layer.

The following two properties of the concatenation function are essential to increase the expressive power of MGNN. First, the output node representation $\mathbf{h}_v^{(l)}$ will not change if the order of concatenation and aggregation is interchanged. Second, $\mathbf{h}_v^{(l)}$ can always explicitly preserve each motif-based feature embedding via the injective combination. Using these two properties, we can theoretically show that MGNN has a larger representation capacity than the standard GNN, as we will further discuss in Section~\ref{sec:theorem}. 


\subsection{Model Training}\label{sec:model_train}
The node representations generated by MGNN can be used for various downstream learning tasks, including supervised and unsupervised learning.

For supervised learning, the node representations can be directly used as features for a specific downstream task, optimized with a supervised loss that can be abstracted as
\begin{equation}\label{equ:suploss}
\mathcal{L}(\mathbf{Y}, \hat{\mathbf{Y}}),
\end{equation}
\begin{equation}\label{equ:predict_func}
\hat{\mathbf{Y}} = \Phi(\mathbf{H}^{(L)}),
\end{equation}
where $\hat{\mathbf{Y}}$ is the predicted matrix. $\mathbf{H}^{(L)}$ is the node representation matrix generated by the last or $L$-th MGNN layer, such that its $i$-th row is the embedding vector $\mathbf{h}^{(L)}_i$ of node $i$ in Eq.~\eqref{equ:concat}. The loss function $\mathcal{L}$, prediction function $\Phi$ and the ground truth $\mathbf{Y}$ depend on the specific downstream task. Taking node classification as an example, the loss can be the cross-entropy loss over the training samples, as follows.
\begin{equation}\label{equ:nodeclsloss}
\sum_{i \in \mathcal{Y}} \sum^{n_c}_{j=1} - (\mathbf{Y})_{ij} \log (\mathbf{\tilde{H}}^{(l)})_{ij},
\end{equation}
where $\mathcal{Y}$ is the set of training node indices,  $n_c$ denotes the number of classes, 
$\mathbf{Y}$ is the ground truth matrix such that its $i$-th row is the one-hot label vector of node $i$, and $\mathbf{\tilde{H}}^{(l)}$ is the predicted matrix such that its $i$-th row is the predicted class distribution of node $i$, which can be obtained by taking a softmax activation or additional neural network layers as the prediction function $\Phi$ and passing $\mathbf{H}^{(l)}$ through $\Phi$.
Another common supervised task is graph classification, which can use a similar cross-entropy loss and prediction function, but the node representations must first undergo a readout operation \cite{DBLP:conf/iclr/XuHLJ19} to generate the graph-level representations.  

For unsupervised learning, the node representations can be trained through the graph auto-encoder  \cite{DBLP:journals/corr/KipfW16a} or other self-supervised frameworks \cite{you2020graph} without any task-specific supervision.


Algorithm \ref{alg:mgnn} summarizes the framework of MGNN. 
To be more specific, MGNN takes the node features, normalized adjacency matrix, and motif-based adjacency matrices for all possible motifs as inputs. The construction of motif-based adjacency matrices is based on our proposed method and another method in the literature mentioned in section \ref{sec:adj}. MGNN further propagates the node representations (or node input features) layer by layer. In each layer, first, from line 4 to line 7, MGNN produces the $M_k$-based representation $\mathbf{h}^{(l)}_{v,k}$ of node $v$ by performing message aggregation. Second, from line 8 to line 14, MGNN compares the $M_k$-based representation with each other using the motif redundancy minimization operator, to distill the features specific to each motif. Third, in line 15, MGNN utilizes the injective concatenation to combine $M_k$-based representations and update the representation of the node. Finally, in line 18, the set of output representations of each node is returned. 

The computational complexity of one MGNN layer is $\mathcal{O}(|\mathcal{V}|^2)$, as follows. Firstly, the complexity is dominated by the computations in Eqs.~\eqref{equ:maf} and \eqref{equ:gcn}, where the time complexities are given by $\mathcal{O}(|\mathcal{V}|d)$ and $\mathcal{O}(|\mathcal{V}|^2d)$, respectively ($d$ denotes the dimension of an MGNN layer). Hence, when computing Eq.~\eqref{equ:maf} over all the nodes, the complexity is $\mathcal{O}(|\mathcal{V}|^2d)$.  Secondly, in our implementation,  Eq.~\eqref{equ:gcn} can be pre-calculated before Eq.~\eqref{equ:maf}. Therefore, the overall time complexity of MGNN is $\mathcal{O}(|\mathcal{V}|^2d)$, which can be further simplified to $\mathcal{O}(|\mathcal{V}|^2)$ as $d$ is typically a small constant.


\begin{algorithm}[t]
 \LinesNumbered 
 \label{alg:mgnn}
 \caption{The framework of MGNN.} 
  \KwIn{Node input feature matrix $\mathbf{X} \in \mathbb{R}^{|\mathcal{V}| \times d_0}$, normalized adjacency matrix $\tilde{\mathbf{A}} \in \mathbb{R}^{|\mathcal{V}| \times |\mathcal{V}|}$, and normalized motif-based adjacency matrices  $\tilde{\mathbf{A}}_{k} \in \mathbb{R}^{|\mathcal{V}| \times |\mathcal{V}|}$, $k \in \{1, ..., 13\}$.}
  \KwOut{Node embedding $\mathbf{h}^{(L)}_v \in \mathbb{R}^{13d_L}$ for each node.}
Randomly initialize all parameters

   $\mathbf{H}^{(0)} \leftarrow \mathbf{X}$
   
   \For{$l = 1, ..., L$}{
   \For{$v \in \mathcal{V}$}{
   

   
   \For{$k = 1, ..., 13$}{
   
  $\mathbf{h}^{(l)}_{v,k} \leftarrow$ Compute the $M_k$-based representation of node $v$ by Eq.~\eqref{equ:maf}
   }
  

     \For{$k = 1, ..., 13$}{

  
  $f(\mathbf{h}^{(l)}_{v,k}) \leftarrow$ Map $\mathbf{h}^{(l)}_{v,k}$ to the same space by Eq.~\eqref{equ:f1}
  
  $\bar{\mathcal{H}}_k \leftarrow$ Concatenate each $\mathbf{h}^{(l)}_{v,i} (i \ne k)$  and $\mathbf{z}_v$ according definition \ref{def:op}
  

  $f_k(\bar{\mathcal{H}}_k) \leftarrow$ Select the redundant feature w.r.t. motif $M_k$ by Eq.~\eqref{equ:f2}

  $\beta_k \leftarrow$ Compute the similarity between $f(\mathbf{h}^{(l)}_{v,k})$ and $f(\bar{\mathcal{H}}_k)$ by Eq.~\eqref{equ:att}
  
  $\tilde{\mathbf{h}}_{v, k}^{(l)} \leftarrow$ Update $\mathbf{h}^{(l)}_{v,k}$ based on $f(\mathbf{h}^{(l)}_{v,k})$, $f_k(\bar{\mathcal{H}}_k)$ and $\alpha_k$ by Eq.~\eqref{equ:instance}

   }
   
  
    $\mathbf{h}^{(l)}_v \leftarrow$ Concatenate each $\tilde{\mathbf{h}}^{(l)}_{v,k}$ by Eq.~\eqref{equ:concat}
   
   }
   }

\Return{ $\{\mathbf{h}^{(L)}_v \big| v \in \mathcal{V}\}$}
\end{algorithm}




\section{Theoretical Analysis}\label{sec:theorem}
In this section, we aim to analyze the representation capacity of MGNN in comparison with standard GNN. In order to facilitate the analysis, we first introduce a simplified version of MGNN, and then further show that even the simplified MGNN still has stronger discriminative power than standard GNN.

\subsection{Simplified Version of MGNN}
A simplified version of the $l$-th MGNN layer is as follows:
\begin{equation}\label{equ:mgnn_abs1}
    \mathbf{h}^{(l)}_{v, k} =  \omega \Big( \Big\{ \alpha^{(l)}_{k,vi} \cdot (\mathbf{A}_k)_{vi}   \mathbf{W}_m^{(l)} \mathbf{h}^{(l-1)}_i  \big| i \in \mathcal{N}(v) \Big\} \Big),
\end{equation}
\begin{equation}\label{equ:mgnn_abs2}
    \mathbf{h}^{(l)}_v =  \big\|_{k=1}^{13} \sigma(\mathbf{h}^{(l)}_{v, k}),
\end{equation}
where $\omega$ represents the aggregate function. 

Then we demonstrate that Eqs.~\eqref{equ:mgnn_abs1}--\eqref{equ:mgnn_abs2} is a simplified version of a MGNN layer. Specifically, in Eq.~\eqref{equ:maf}, normalized $\tilde{\mathbf{A}}_k$ and $(\mathbf{Z}^{(l)})_{i*}$ are substituted for $\mathbf{A}_k$ as well as $\mathbf{W}_m^{(l)}\mathbf{h}^{(l-1)}_i$, respectively, where $\mathbf{W}_m^{(l)}$ is the trainable weight matrix in the $l$-th simplified MGNN layer and $\mathbf{h}^{(l-1)}_i$ is the node messages from the previous ($l-1$)-th simplified MGNN layer ($\mathbf{h}^{(0)}_i = \mathbf{x}_i$). After that, Eq.~\eqref{equ:mgnn_abs1} is obtained. Then, the output $\mathbf{h}^{(l)}_{v, k}$ of Eq.~\eqref{equ:mgnn_abs1} is utilized in place of $\tilde{\mathbf{h}}_{v, k}^{(l)}$ in Eq.~\eqref{equ:concat} and then Eq.~\eqref{equ:mgnn_abs2} is obtained. Thus, Eqs.~\eqref{equ:mgnn_abs1}--\eqref{equ:mgnn_abs2} is a simplified version of a MGNN layer.




\begin{table}[htbp]
	\centering
	\caption{The layer of the abstract model of the standard GNN or the simplified version of MGNN.}
	\begin{tabular}{lc}
		\toprule
		Step1 & message aggregation \\
		\midrule
		Standard GNN & $\bar{\mathbf{h}}^{(l)}_v = \omega \left(\left\{ (\mathbf{A})_{vi} \mathbf{W}_s^{(l)} \tilde{\mathbf{h}}^{(l-1)}_i \big| i \in \mathcal{N}(v) \right\}\right)$\\
		MGNN &      $\mathbf{h}^{(l)}_{v, k} =  \omega \Big( \Big\{\alpha^{(l)}_{k,vi} \cdot (\mathbf{A}_k)_{vi}   \mathbf{W}_m^{(l)}\mathbf{h}^{(l-1)}_i  \big| i \in \mathcal{N}(v) \Big\} \Big)$\\
		\midrule
		Step2 & node representation update \\
		\midrule
		Standard GNN & $\tilde{\mathbf{h}}^{(l)}_v = \sigma(\bar{\mathbf{h}}^{(l)}_v) $ \\
		MGNN &  $\mathbf{h}^{(l)}_v =  \big\|_{k=1}^{13} \sigma(\mathbf{h}^{(l)}_{v, k})$\\
		\bottomrule
	\end{tabular}%
	\label{tab:architecture_comparison}%
\end{table}%

\subsection{Representational Capacity Study}
In order to  compare the representational capacity of the simplified version of MGNN with that of the standard GNN, we begin with the layers of the abstract model of the standard GNN and the simplified version of MGNN in Table \ref{tab:architecture_comparison}, where $\mathbf{W}_s^{(l)}$ is the trainable weight matrix in the $l$-th standard GNN layer, and $\tilde{\mathbf{h}}^{(l-1)}_i$ is the node messages from the previous ($l-1$)-th standard GNN layer ($\tilde{\mathbf{h}}^{(0)}_i = \mathbf{x}_i$). The mainstream models of GNNs, including GCN, GAT, GraphSAGE and GIN, can be viewed as an instance of the standard GNN. Then we shows MGNN has larger representational capacity than standard GNN in Lemmas~\ref{lma:special_case}--\ref{lma: example} and Theorem \ref{thm:powerful}. 

Based on the above abstractions, we first show that even a special case of MGNN at least has the same representational capacity as the standard GNN in Lemma \ref{lma:special_case}.

\begin{lemma}\label{lma:special_case}
	Given any an instance of the standard GNN, if the aggregate functions of standard GNN and MGNN are the same and the input to $\omega$ only consist of values in the same dimension from different feature vectors, its representations of the graphs can also be generated by a special case of MGNN. 
\end{lemma}
The proof for Lemma \ref{lma:special_case} is hinged on two important properties of the injective concatenation function, i.e., the interchangeability of concatenation and aggregation, and the explicit preservation of motif-based representations, as first mentioned in Section~\ref{sec:update}. The detailed proof can be found in Section I of our supplementary materials. In short, Lemma \ref{lma:special_case} shows that a standard GNN can be subsumed by MGNN. Taking one step further, we show that there exist two graphs that can be distinguished by MGNN but are indistinguishable by the standard GNN.

\begin{lemma}\label{lma: example}
There exist two non-isomorphic graphs $G$ and $G'$ with self-loops, which can be distinguished by MGNN, but not by the standard GNN.
\end{lemma}
\begin{proof}

\begin{figure}[t!]
	\centering
	\includegraphics[width=0.65\columnwidth]{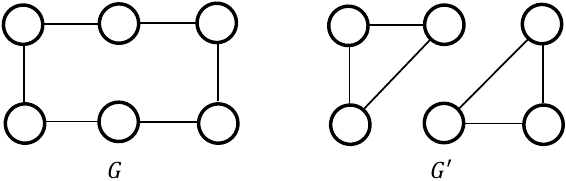} 
	\caption{Two graphs with self-loops that cannot be distinguished by the standard GNN. Inside these two graphs, the features of the nodes are the same and the self-loops are not depicted for brevity.}
	\label{fig:lemma2}
\end{figure}

As Fig.~\ref{fig:lemma2} illustrates, consider the two non-isomorphic graphs $G$ and $G'$ with self-loops, in which all nodes have the same features.  First, $G$ and $G'$ cannot be distinguished by standard GNN, because the multi-set of neighboring features of each node are the same. Second, $G$ and $G'$ can be naturally distinguished by MGNN since each node in $G$ is only associated with an open motif, whereas each node in $G'$ is associated with both open and close motifs. 
\end{proof}

\begin{table*}[htbp]
  \centering
  \caption{Statistics of the datasets.}
    \begin{tabular}{r|lrlccc}
    \toprule
    \multicolumn{1}{l}{Category} & Dataset & \multicolumn{1}{c}{\# Graphs} & \multicolumn{1}{c}{\# Nodes (Avg.)} & \# Edges (Avg.) & \# Features & \# Classes \\
    \midrule
    \midrule
    \multicolumn{1}{l|}{\multirow{3}[2]{*}{Citation Graphs}} & Cora  & \multicolumn{1}{c}{1} & \multicolumn{1}{c}{2,708} & 5,429 & 1,433 & 7 \\
          & Citeseer & \multicolumn{1}{c}{1} & \multicolumn{1}{c}{3,327} & 4,732 & 3,703 & 6 \\
          & Pubmed & \multicolumn{1}{c}{1} & \multicolumn{1}{c}{19,717} & 44,338 & 500   & 3 \\
    \midrule
    \multicolumn{1}{l|}{Knowledge Graphs} & Chem2Bio2RF & \multicolumn{1}{c}{1} & \multicolumn{1}{c}{295,911} & 727,997 & 5     & 10 \\
    \midrule
    \multicolumn{1}{l|}{\multirow{3}[2]{*}{Biochemical  Graphs}} & MUTAG & \multicolumn{1}{c}{188} & \multicolumn{1}{c}{17.90 } & 19.79 & 7     & 2 \\
          & ENZYMES & \multicolumn{1}{c}{600} & \multicolumn{1}{c}{32.63} & 62.14 & 21    & 6 \\
          & AIDS  & \multicolumn{1}{c}{2,000} & \multicolumn{1}{c}{15.69} & 16.20  & 42    & 2 \\
    \midrule
    \bottomrule
    \end{tabular}%
  \label{tab:dataset}%
\end{table*}%

Based on the results on Lemma~\ref{lma:special_case} and Lemma \ref{lma: example}, we immediately come to the conclusion about the representation capacity in Theorem~\ref{thm:powerful}.
\begin{theorem}\label{thm:powerful}
	MGNN has a larger representation capacity than the standard GNN.
\end{theorem}

\section{Experiments}\label{sec:experiments}
In this section, we introduce the details of the experimental setup and the comparison results.

\subsection{Experimental Setup}
\subsubsection{Datasets} To evaluate the effectiveness of our proposed MGNN, we utilize seven public datasets on two benchmark tasks: (1) classifying nodes on three citation network datasets (Cora, Citeseer, and Pubmed) and a knowledge graph (Chem2Bio2RDF), and (2) classifying graphs on three biochemical graph datasets (MUTAG, ENZYMES, and AIDS). Table \ref{tab:dataset} summarizes the statistics of seven datasets.
\begin{itemize}
    \item \textbf{Cora, Citeseer} and \textbf{Pubmed}~\cite{DBLP:journals/aim/SenNBGGE08} contain documents represented by nodes and citation links represented by edges.
    \item \textbf{Chem2Bio2RDF}~\cite{chen2010chem2bio2rdf} integrates data from multiple public sources. Because the node feature is not provided in Chem2Bio2RDF and the discriminative power of GNN-based methods often depends on the properties of nodes, we use the degree statistical information of each node and its $1$-hop neighborhood (5 dimensions in total) \cite{DBLP:journals/corr/abs-1811-03508} as its node features. 
    \item \textbf{MUTAG}~\cite{DBLP:conf/icml/KriegeM12} contains 188 chemical compounds divided into two classes according to their mutagenic effect on a bacterium.
    \item \textbf{ENZYMES}~\cite{DBLP:conf/ismb/BorgwardtOSVSK05} contains 100 proteins from each of the 6 Enzyme Commission top level enzyme classes.
    \item \textbf{AIDS}\footnote{https://wiki.nci.nih.gov/display/NCIDTPdata/AIDS+Antiviral+Screen+Data/}~\cite{DBLP:conf/sspr/RiesenB08} contains 2 classes (active, inactive), which represent molecules with activity against HIV or not.  
\end{itemize}


\subsubsection{Baselines} We consider three categories of methods, namely low- and high-order GNN-based methods, network embedding-based methods, as well as graph pooling-based methods. Low-order GNN-based methods include:
\begin{itemize}
	\item \textbf{GCN} \cite{DBLP:conf/iclr/KipfW17} aggregates the feature information from a node’s neighborhood.
	\item \textbf{GraphSAGE}~\cite{hamilton2017inductive} generates embeddings by sampling and aggregating features from a node’s local neighborhood.
	\item \textbf{GAT}~\cite{DBLP:journals/corr/abs-1710-10903}  incorporates the attention mechanism into the propagation step, following a self-attention strategy.
	\item \textbf{GIN}~\cite{DBLP:conf/iclr/XuHLJ19} performs the feature aggregation in an injective manner based on the theory of the $1$-WL graph isomorphism test.
    \item \textbf{BGNN} \cite{DBLP:conf/iclr/0004P21} combines gradient boosted decision trees (GBDT) with GNN.
\end{itemize} 

High-order GNN-based methods include:
\begin{itemize}
    \item \textbf{MotifNet} \cite{DBLP:conf/dsw/MontiOB18} utilizes a Laplacian matrix based on multiple motif-based adjacency matrices as the convolution kernel of the graph, and uses an attention mechanism to select node features.
    \item \textbf{MixHop} \cite{abu2019mixhop} concatenates the aggregated node features from neighbors at different hops in each layer.
    \item \textbf{GDC} \cite{klicpera2019diffusion} utilizes generalized graph diffusion (e.g. Personalized PageRank) to generate a new graph, then uses this new graph to predict rather than the original graph.
    \item \textbf{CADNet} \cite{lim2021class} obtains neighborhood representations by random walks with attention, and incorporates the neighborhood representations via trainable coefficients.
\end{itemize} 

Network embedding-based methods include:
\begin{itemize}	
	\item \textbf{DeepWalk} \cite{DBLP:conf/kdd/PerozziAS14} combines truncated random-walk with skip-gram model to learn node embedding.
	\item \textbf{GraRep}~\cite{DBLP:conf/cikm/CaoLX15} leverages various powers of the adjacency matrix  to capture higher-order node similarity.
	\item \textbf{HOPE}~\cite{DBLP:conf/kdd/OuCPZ016} preserves higher-order proximity in node representations.
    \item \textbf{Node2Vec} \cite{DBLP:conf/kdd/GroverL16} employs biased-random walks, which provide a trade-off between breadth-first (BFS) and depth-first (DFS) graph searches, to learn node embedding.
	\item \textbf{Graph2Vec}~\cite{DBLP:journals/corr/NarayananCVCLJ17} creates WL tree for nodes as features in graphs to decompose the graph-feature co-occurence matrix.
	\item \textbf{NetLSD}~\cite{DBLP:conf/kdd/TsitsulinMKBM18} calculates the heat kernel trace of the normalized Laplacian matrix over a vector of time scales.
	\item \textbf{GL2Vec}~\cite{DBLP:conf/iconip/ChenK19} extends Graph2Vec with edge features by utilizing the line graph.
	\item \textbf{Feather}~\cite{DBLP:conf/cikm/RozemberczkiS20} describes
    node neighborhoods with random walk weights. 
\end{itemize}  

Graph pooling-based methods include:
\begin{itemize}	
    \item \textbf{Graclus}~\cite{DBLP:journals/pami/DhillonGK07} is an alternative of eigen-decomposition to calculate a clustering version of the original graph.
	\item \textbf{GlobalATT}~\cite{DBLP:journals/corr/LiTBZ15} employs gate recurrent unit architectures with global attention to update node latent representations. 
    \item \textbf{EdgePool} \cite{diehl2019towards} extracts graph features by contracting edges and merging the connected nodes uniformly.
    \item \textbf{TopKPool}~\cite{DBLP:conf/icml/GaoJ19} learns a scalar projection score for each node and selects the top $k$ nodes.
    \item \textbf{ASAP}~\cite{DBLP:conf/aaai/RanjanST20} utilizes a self-attention network along with a modified GNN formulation to capture the importance of each node in a given graph.
\end{itemize}



We use the low- and high-order GNN-based approaches for both node classification and graph classification tasks, except that BGNN is used for the node classification task only since its GNN module is designed to provide the gradients generated by the node classification loss during training \cite{DBLP:conf/iclr/0004P21}. We use the following node-level network embedding approaches for the node classification task only: DeepWalk \cite{DBLP:conf/kdd/PerozziAS14}, GraRep~\cite{DBLP:conf/cikm/CaoLX15}, HOPE~\cite{DBLP:conf/kdd/OuCPZ016}, and Node2Vec \cite{DBLP:conf/kdd/GroverL16}. We use all the graph pooling approaches and the following graph-level network embedding approaches for the graph classification task only: Graph2Vec~\cite{DBLP:journals/corr/NarayananCVCLJ17}, NetLSD~\cite{DBLP:conf/kdd/TsitsulinMKBM18}, GL2Vec~\cite{DBLP:conf/iconip/ChenK19}, and Feather~\cite{DBLP:conf/cikm/RozemberczkiS20}.

\subsubsection{Implementation details} 

The configurations of our MGNN as well as GNN-based baselines on the node classification task are as follows. We use 1 GNN layer for Cora and Citeseer datasets, while 2 GNN layers for the other two larger datasets namely Pubmed and Chem2Bio2RDF. In addition, a fully connected layer (FCL) is added after the last GNN layer to further process the node representation matrix. For the graph classification task, we use 3 GNN layers on 3 biochemical graph datasets for MGNN, GNN-based baselines as well as graph pooling-based baselines. Similarly, the node representation matrix after the last GNN layer would be passed through three fully connected layers. We use sum aggregation as the readout operation to derive the embedding for the graph.

For our MGNN, aggregate function AGG in Eq.~\eqref{equ:maf} was sum. The  activation function $\sigma$ in Eq.~\eqref{equ:att} was set as sigmoid for Cora and Citeseer, while it was set as tanh for other datasets \cite{DBLP:journals/corr/abs-1710-10903}. We further set $d_1, d_2, d_3$ in Eq.~\eqref{equ:maf}, the output dimensionality of the GCN layer which is stacked in the first, second, and third MGNN layers, to $16, n_c, n_c$, respectively, where $n_c$ is the number of classes in the corresponding dataset. Next, the dimensionality $d'_l$ in Eq.~\eqref{equ:f1} was set to 6 on each dataset.
We used the Adam optimizer and the learning rate $\eta$ in the optimization algorithm was set as 0.011. The maximum number of training epochs $t$ was set as 3000. In practice, we made use of PyTorch for an efficient GPU-based implementation of Algorithm \ref{alg:mgnn} using sparse-dense matrix multiplications.\footnote{Our source codes and pre-processed datasets are publicly available via https://github.com/DMIRLAB-Group/MGNN}

\begin{table*}[htbp]
  \centering
  \caption{Performance on the node classification task, measured in accuracy. Standard deviation errors are given. The best performance is marked in bold, and the second best is underlined.}

    \begin{tabular}{lcccc}
    \toprule
            & Cora  & Citeseer   & Pubmed  & Chem2Bio2RDF\\
    \midrule
    DeepWalk       &   0.4313 $\pm$ 0.0221    &   0.2732 $\pm$ 0.0216  &  0.4440 $\pm$ 0.0208   & 0.9253  $\pm$ 0.0023\\
    GraRep         &   0.5957 $\pm$ 0.0062    &  0.4220 $\pm$ 0.0022   &  0.6147 $\pm$ 0.0073   &  0.9313 $\pm$ 0.0018\\
    HOPE          &   0.4510 $\pm$ 0.0010    &   0.3180 $\pm$ 0.0021   &  0.4880 $\pm$ 0.0011  & 0.9030  $\pm$ 0.0001\\
    Node2Vec      &  0.7150 $\pm$ 0.0042    &   0.4670 $\pm$ 0.0145    & 0.6788  $\pm$ 0.0063 & 0.9029  $\pm$ 0.0012\\
    \midrule
    GCN            & 0.8595 $\pm$ 0.0207 &   0.7764 $\pm$ 0.0045  & 0.8865 $\pm$ 0.0048   & 0.9371 $\pm$ 0.0017 \\
    GraphSAGE    &   0.8610 $\pm$ 0.0101    &  0.7744 $\pm$ 0.0061   &  \underline{0.8980} $\pm$ 0.0049   & \underline{0.9630} $\pm$ 0.0010\\
    GAT         & 0.8775 $\pm$ 0.0127 &   0.7852 $\pm$ 0.0052   &  0.8840 $\pm$ 0.0079   & 0.9628 $\pm$ 0.0017\\
    GIN         &    0.8107 $\pm$ 0.0188  &   0.7255 $\pm$ 0.0160   &  0.8810 $\pm$ 0.0156  & 0.9205 $\pm$ 0.0129\\
    BGNN      &   0.8470 $\pm$ 0.0143  &    0.7750 $\pm$ 0.0112    &  0.8380 $\pm$ 0.0119    &  0.8746 $\pm$ 0.0115  \\
    \midrule
    MotifNet     & 0.8580 $\pm$ 0.0075 &   0.7750 $\pm$ 0.0071   &  0.8895 $\pm$ 0.0102  & 0.8863 $\pm$ 0.0114\\
    MixHop   &   \underline{0.8803} $\pm$ 0.0120   &     0.7796 $\pm$ 0.0053    &   0.8628 $\pm$ 0.0150   &  \underline{0.9630} $\pm$ 0.0004 \\
    GDC      &  0.8660 $\pm$ 0.0100  &    \underline{0.7854} $\pm$ 0.0061   &   0.8768 $\pm$ 0.0059    &  0.8838 $\pm$ 0.0036 \\
    CADNet    &  0.8612 $\pm$ 0.0131  &    0.7652 $\pm$ 0.0148     &   0.8772 $\pm$ 0.0085    &  0.8287 $\pm$ 0.0258 \\
    \midrule
    MGNN          & \textbf{0.9060} $\pm$ 0.0049 &   \textbf{0.7948} $\pm$ 0.0050  & \textbf{0.9232} $\pm$ 0.0084    & \textbf{0.9870} $\pm$ 0.0021\\
    \bottomrule
    \end{tabular}%
  \label{tab:baselines_ncls}%
\end{table*}%


For the baselines, we tuned their settings empirically. First, for GNN-based methods and graph pooling-based methods, the embedding dimension and dropout \cite{DBLP:journals/corr/abs-1207-0580} rate of these models, were set to 16, 0.5, respectively. GCN, GAT, MotifNet used their default aggregate function and GraphSAGE used max aggregate empirically. The degree of multivariate polynomial filters in MotifNet was set to 1 and utilizes 13 motif-based  adjacency matrices. Considering that GAT concatenates different head outputs, which is similar to MGNN. Therefore, GAT was set to use 13 heads and the embedding dimension is 8. Second, for network embedding-based methods, the embedding dimension of these models was set to 128, and we used the logistic regression model \cite{cox1958regression} as a classifier to evaluate the quality of the embeddings generated by these unsupervised models. 
The other settings for these models largely align with the literature.

Note that, in our experiments, all the methods make use of the same directed/undirected edge information on each dataset. Specifically, Chem2Bio2RDF is a directed graph. The baseline implementations used here are able to deal with directed graphs, in which message propagation follows the given edge directions. Meanwhile, ENZYMES, MUTAG, and AIDS are all undirected graphs, and the original Cora, Citeseer and Pubmed are directed citation graphs. Following standard benchmarking practice, the three citation graphs are treated as undirected 
\cite{wu2020comprehensive,yang2016revisiting}, where a preprocessing step is applied to ignore edge directions for all the methods.  
Note that when an undirected graph is fed into MGNN, MGNN treats each undirected edge as two directed edges in opposite directions.

We adopt a widely-used \emph{accuracy} metric for performance evaluation.  For the node classification task, similar to the experimental setup in \cite{DBLP:conf/iclr/ChenMX18}, we split the dataset into 500 nodes for validation, 500 nodes for testing, and the remaining nodes were used for training, to simulate labeled and unlabeled information. Note that Chem2Bio2RDF is an exception, and we split it into 5000 nodes for validation and 5000 nodes for testing due to its large size. Then we report the average and standard deviation of accuracy scores across the 5 runs with different random seeds. For the graph classification task, similar to the experimental setup in \cite{DBLP:conf/iclr/XuHLJ19}, we perform 5-fold cross validation. For other experiments, we present the average accuracy scores over the 5 runs with various random seeds.

\subsection{Performance Evaluation}
We evaluate the empirical performance of MGNN against the state-of-the-art baselines in Tables ~\ref{tab:baselines_ncls} and \ref{tab:baselines_gcls}.

\subsubsection{\textbf{Comparison to baselines}} 

As shown in Table~\ref{tab:baselines_ncls}, MGNN significantly and consistently outperforms all the baselines on different datasets. In particular, GraphSAGE achieves the second best performance on Pubmed and Chem2Bio2RDF, while MixHop achieves the second best performance on Cora and Chem2Bio2RDF, and GDC achieves the second best performance on Citeseer. Our MGNN is capable of achieving further improvements against GraphSAGE by 2.81\% on Pubmed, against MixHop by 2.92\% on Cora, as well as against GraphSAGE and MixHop by 2.49\% on Chem2Bio2RDF. On Citeseer, MGNN outperforms GDC by 1.20\% in terms of accuracy. Note that the number of edges in Citeseer is small and the occurrences of motifs are limited. Therefore, our MGNN cannot collect as much high-order information as it can on other datasets, and MGNN achieves less improvement on Citeseer than on other datasets.


Similarly, in Table~\ref{tab:baselines_gcls}, MGNN regularly surpasses all the baselines. In particular, GCN achieves the second best performance on AIDS, while GDC achieves the second best performance on MUTAG, and MixHop achieves the second best performance on ENZYMES. MGNN is able to achieve further improvements against GCN by 0.76\% on AIDS, against GDC by 3.18\% on MUTAG and against MixHop by 10.95\% on ENZYMES as shown in Table~\ref{tab:baselines_gcls}. In particular, a graph represents a compound's molecular structure in these three biochemical graph datasets. Any chemical structure can be represented by 13 motifs, which allows our MGNN to identify similar structures among various compounds and boost classification accuracy. For example, both carbon dioxide $CO_2$ and methane $CH_4$ have the motif $M_8$. Moreover, $CH_4$ has six $M_8$ while $CO_2$ has one $M_8$ only, and such difference is useful for graph classification.

\begin{table}[htbp]
  \centering
  \caption{Performance on the graph classification task in terms of accuracy. Standard deviation errors are given.}
  

    \begin{tabular}{lccc}
    \toprule
           & MUTAG & ENZYMES  & AIDS \\
    \midrule
    Graph2Vec        &    0.6650 $\pm$ 0.0087  &   0.2033 $\pm$ 0.0239    &  0.8045 $\pm$ 0.0033\\
    GL2Vec            &    0.6703 $\pm$ 0.0106  &   0.1967 $\pm$ 0.0461 &  0.8225 $\pm$ 0.0065 \\
    NetLSD         &    0.7450 $\pm$ 0.0611  &   0.2136 $\pm$ 0.0461   &  0.9575 $\pm$ 0.0082 \\
    Feather          &    0.7716 $\pm$ 0.0341  &   0.2483 $\pm$ 0.0226  &  0.7930 $\pm$ 0.0019 \\
    \midrule
    Graclus           &    0.7504 $\pm$ 0.0750 &   0.2567 $\pm$ 0.0253  & 0.8640 $\pm$ 0.0398 \\
    ASAP             &    0.7562 $\pm$ 0.0799  &   0.2600 $\pm$ 0.0320 & 0.8960 $\pm$ 0.0279 \\
    EdgePool          &   0.7508 $\pm$ 0.0687 &   0.2500 $\pm$ 0.0449  & 0.8615 $\pm$ 0.0581 \\
    TopKPool        &    0.7238 $\pm$ 0.0527  &   0.2417 $\pm$ 0.0349 &  0.8530 $\pm$ 0.0492\\
    GlobalATT         &   0.7346 $\pm$ 0.0736  &   0.2383 $\pm$ 0.0427  & 0.8390 $\pm$ 0.0248 \\
    \midrule
    GCN             &    0.7555 $\pm$ 0.0651  &   0.2100 $\pm$ 0.0285  &  \underline{0.9895} $\pm$ 0.0091 \\
    GAT               &   0.7391 $\pm$ 0.0315 &   0.1667 $\pm$ 0.0000   &  0.8740 $\pm$ 0.1013 \\
    GraphSAGE       &    0.7984 $\pm$ 0.0526 &   0.2333 $\pm$ 0.0586  & 0.9855 $\pm$ 0.0091 \\
    GIN              & 0.7780 $\pm$ 0.0940     &   0.2630 $\pm$ 0.0330      & 0.9870 $\pm$ 0.0090\\
    \midrule
    MotifNet       &   0.8040 $\pm$ 0.0330 &   0.1770 $\pm$ 0.0140    &  0.9880 $\pm$ 0.0060 \\
    MixHop          &  0.7663 $\pm$ 0.0897   &  \underline{0.2767} $\pm$ 0.0494    &  0.9265 $\pm$ 0.0157\\
    GDC           &    \underline{0.8199} $\pm$ 0.0849  &  0.2633 $\pm$ 0.0126    & 0.8705 $\pm$ 0.0165 \\
    CADNet       &   0.7450 $\pm$ 0.0531  & 0.2267 $\pm$ 0.0273      & 0.7995 $\pm$ 0.0011 \\
    \midrule
    MGNN    &   \textbf{0.8460} $\pm$ 0.0230 & \textbf{0.3070} $\pm$ 0.0300   & \textbf{0.9970} $\pm$ 0.0030 \\
    \bottomrule
    \end{tabular}%
  \label{tab:baselines_gcls}%
\end{table}%

Next, we further compare the robustness of MGNN and the baseline approaches by introducing noise. Specifically, by replacing the original input node features with a 16-dimensional random vector, we first modified Cora and Pubmed which are denoted as Cora-RandomX and Pubmed-RandomX. Then we compared the performance of MGNN and the baselines on the above two modified datasets. As shown in Fig.~\ref{fig:ablation_x}(a)(b), the performance of MGNN and other GNN baselines on Cora-RandomX and Pubmed-RandomX showed signs of deterioration to varying degrees compared with that on Cora and Pubmed, which is intuitive since additional noise is introduced through random features. Importantly, not only does MGNN considerably and continuously exceed all GNN baselines on the accuracy metric but also its rate of decrease is the lowest. This is because MGNN can grasp more high-order structure information with higher discriminative power, which makes MGNN more robust than other standard or motif-based GNNs.

\begin{figure}[t!]
	\centering
	\includegraphics[width=\columnwidth]{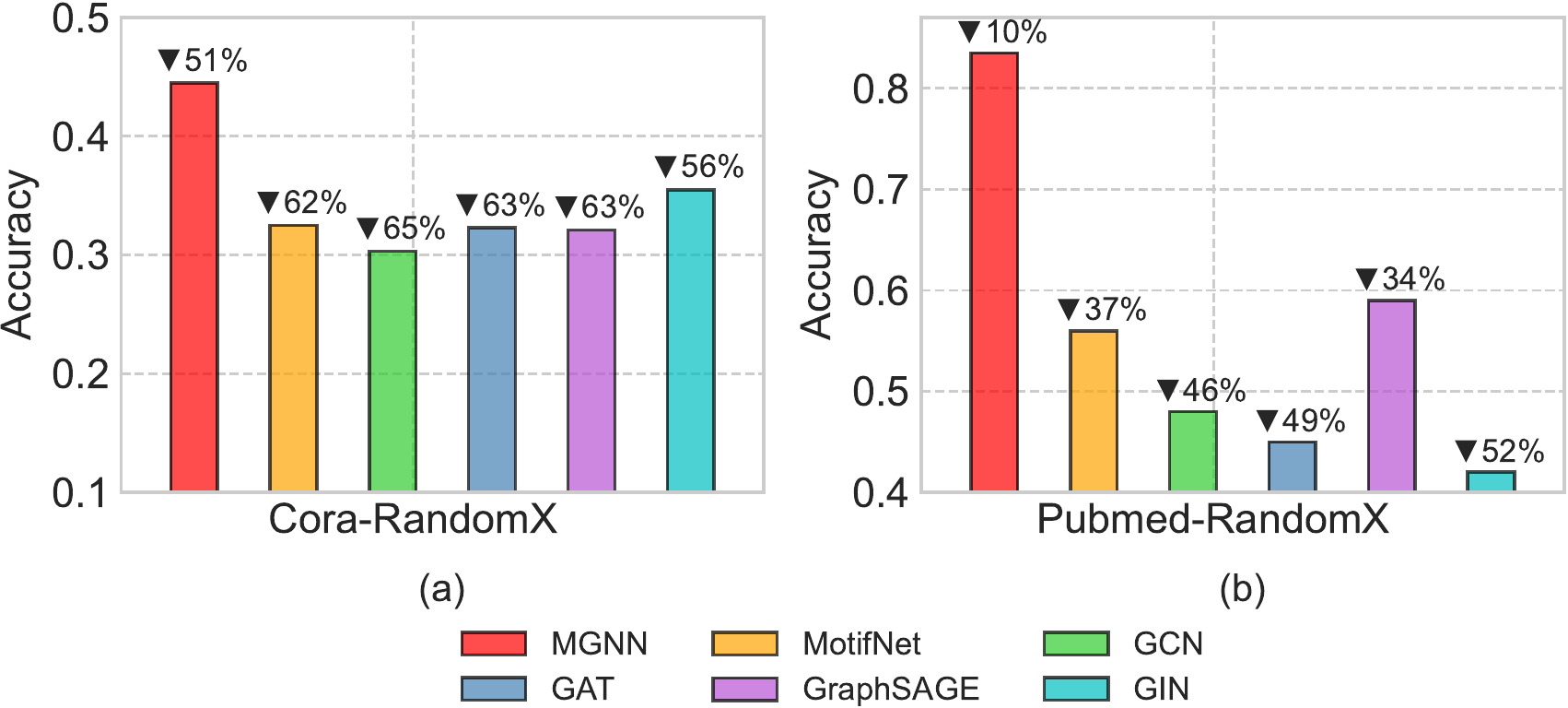} 
	\caption{Performance on robustness comparison in two modified datasets using random vectors as node features. Compared to Cora and Pubmed, the performance degradation rate of models is denoted by the inverted black triangle with the percentage on the top of each bar.}
	\label{fig:ablation_x}
\end{figure}


\subsubsection{Model ablation study}
\begin{figure}
    \centering
    \subfigure[]{
        \includegraphics[width=\columnwidth]{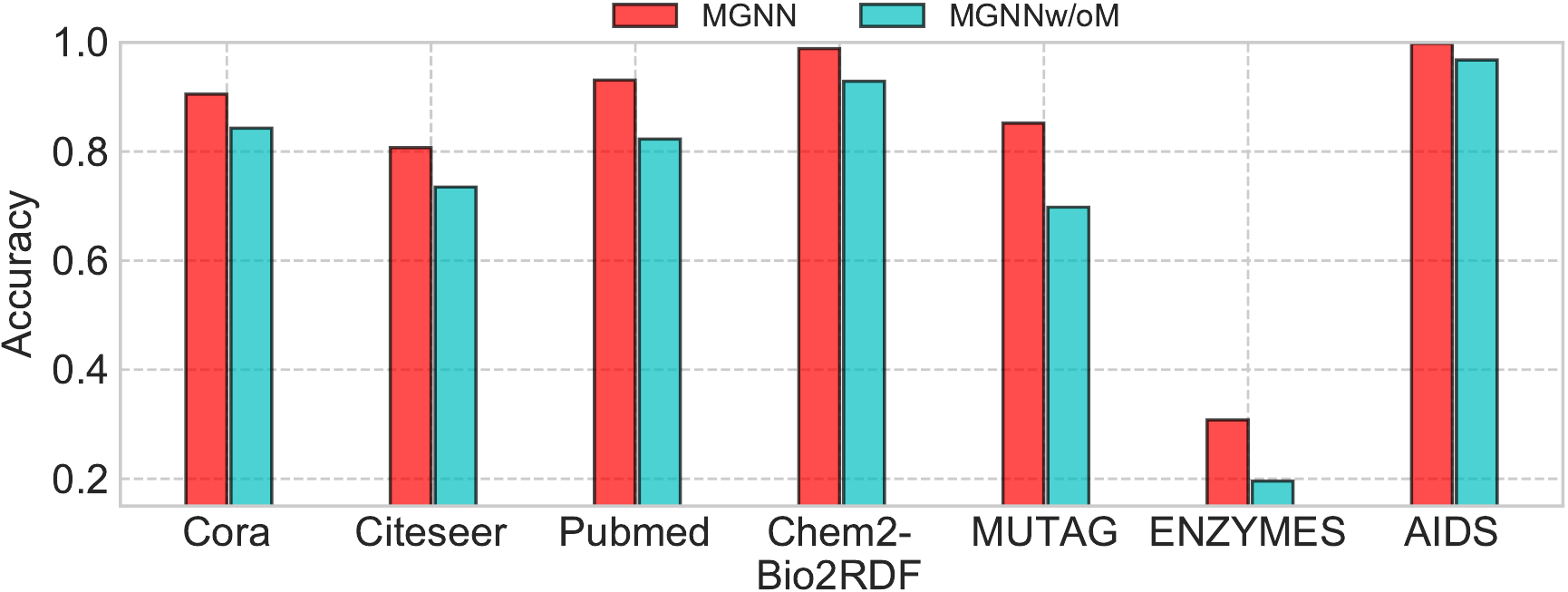}
        \label{fig:ablation_motif}
    }  
    \subfigure[]{
    \includegraphics[width=0.466\columnwidth]{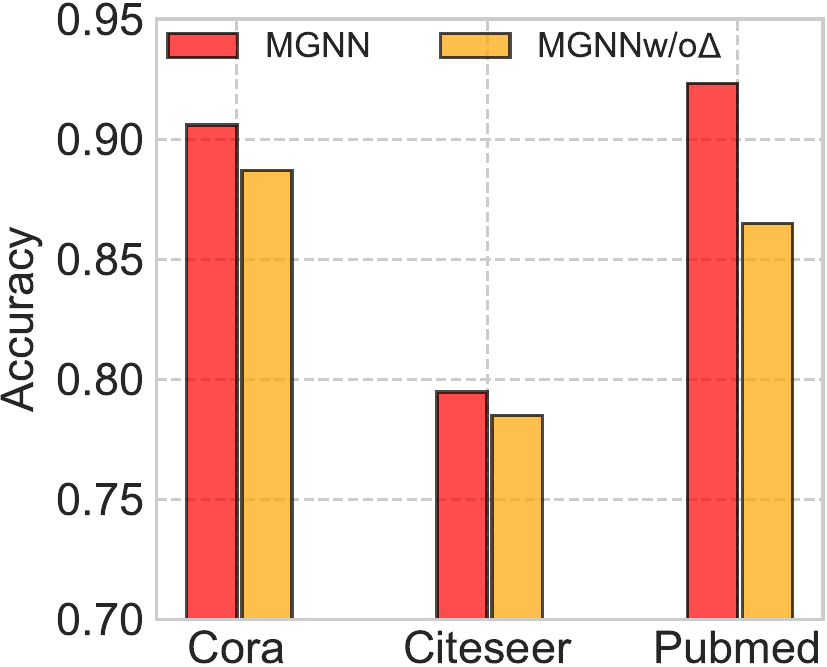}
    \label{fig:ablation_delta}
    }
    \subfigure[]{
    \includegraphics[width=0.466\columnwidth]{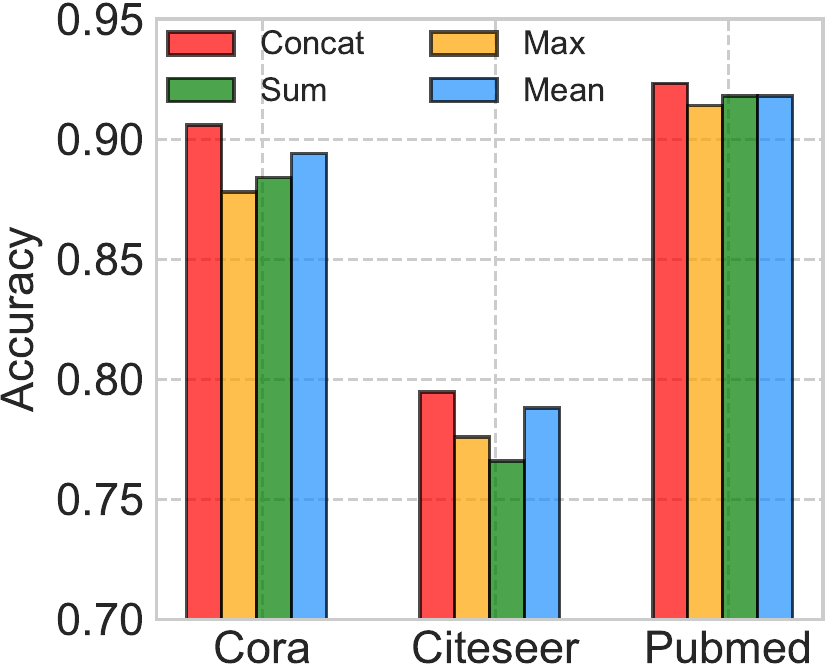}
    \label{fig:ablation_concat}
    }
    \caption{Ablation study on all motifs on seven datasets, minimal-redundancy operator $\Delta$ and injective vector concatenation function on three datasets: (a) the ablation study result related to all motifs on seven datasets and (b)(c): the ablation study results related to $\Delta$ and concatenation function on three datasets, respectively.}
    \label{fig:ablation_mgnn}
\end{figure}

As Fig. \ref{fig:model} illustrates, the network motif, motif redundancy minimization and injective $M_k$-based representation concatenation are key components in our proposed MGNN. We can thus derive the following variants of MGNN: (1) MGNN without any motif information denoted as MGNNw/oM (this variant is actually a GCN); (2) MGNN without motif redundancy minimization operator $\Delta$ denoted as MGNNw/o$\Delta$; (3) MGNN with other functions for feature vector combination, including summation, max and mean. To show the impact of motifs, $\Delta$ and injective concatenation in MGNN, we compare MGNN with the above variants.

In Fig. \ref{fig:ablation_mgnn}, we observe that MGNN achieves better performance than the variants in terms of accuracy, demonstrating the effectiveness of motifs, $\Delta$ and injective concatenation. Firstly, in order to demonstrate the impact of the motif on MGNN's performance, we compare MGNN with the variant MGNNw/oM. As Fig.~\ref{fig:ablation_motif} shows, we observe that the performance of MGNNw/oM is lower than that of MGNN on all the 7 datasets, demonstrating the importance of incorporating motif into GNNs, that is, the high-order structures grasped by motif are important for GNNs' performance. Secondly, in order to investigate the impact of $\Delta$, we removed the minimal-redundancy operator $\Delta$ on MGNN, and the comparison between MGNN and MGNNw/o$\Delta$ is as Fig.~\ref{fig:ablation_delta} illustrates. As can be seen, MGNN significantly outperforms MGNNw/o$\Delta$ in terms of accuracy in three datasets. Note that the number of edges in Pubmed is large and the redundancy of motifs is probably higher than other datasets (i.e, different motifs in Pubmed share more certain substructures). Therefore, on Pubmed, MGNNw/o$\Delta$ is more difficult to distinguish between different motif-wise representations than on other datasets, and the performance gap between MGNN and MGNNw/o$\Delta$ is more pronounced on Pubmed than on other datasets. Thirdly, to demonstrate the impact of injective concatenation, we used other non-injective vector combination functions, including summation, max, and mean, to replace injective concatenation. Fig.~\ref{fig:ablation_concat} illustrates that MGNN with concatenation performs significantly better than MGNN with other functions on these three datasets. Moreover, we show the results of different combination functions (namely, concatenation, max, sum and mean) on three datasets in Table \ref{tab:ablation_mgnn}. We can observe that the performance decline is larger on Citeseer and Cora than on PubMed. A potential reason is Cora and Citeseer are very sparse and the occurrences of motifs are limited. Thus, the limited number of motifs on Cora and Citeseer would make it more difficult to distinguish among different node representations using non-injective functions.

\begin{table}[htbp]
  \centering
  \addtolength{\tabcolsep}{-1mm}
  \caption{Performance of MGNN by using different combination functions on three datasets. The rates of decline in performance w.r.t.~concatenation are given in parentheses.}  \label{tab:ablation_mgnn}
    \begin{tabular}{lcccc}
    \toprule
          & Concat & Max   & Sum   & Mean \\
    \midrule
    Cora  & 0.906 & 0.878 (3.09\% $\downarrow$) & 0.884 (2.43\% $\downarrow$) & 0.894 (1.32\% $\downarrow$) \\
    Citeseer & 0.795 & 0.776 (2.37\% $\downarrow$) & 0.766 (3.62\% $\downarrow$) & 0.788 (0.86\% $\downarrow$) \\
    Pubmed & 0.923 & 0.914 (1.00\% $\downarrow$) & 0.918 (0.56\% $\downarrow$) &  0.918 (0.56\% $\downarrow$) \\
    \bottomrule
    \end{tabular}%
\end{table}%

\subsection{Case Study}\label{sec:case}
In this section, we investigate the importance of different motifs for prediction and demonstrate the necessity of using the high-order structure for prediction. We completed the following two studies on the Chem2Bio2RDF dataset. First, MGNN makes a prediction across the 5 runs by using only 1 out of the 13 motifs and then compares the results to determine the significance of various motifs. Second, we take protein-disease association prediction as our case study. In particular, we rank the protein-disease pairs based on their predicted scores, then identify those top pairs supported by existing publications. Meanwhile, we also evaluate the performance of MGNN versus the baseline approaches for protein-disease association prediction. Next, we show the details of these two studies. 

\subsubsection{The importance of different motifs for prediction} To demonstrate the importance of different motifs, MGNN utilizes just one motif to conduct node classification across the 5 runs on the Chem2Bio2RDF network, and the performance of MGNN is used to assess the importance of each motif in Table \ref{tab:important_motif}.

\begin{table}[htbp]
  \centering
  \caption{Importance ranking of 13 motifs on Chem2Bio2RDF. The importance score of each motif is the performance of MGNN when using only that motif for prediction. The symbol `ALL' indicates that all motifs are used by MGNN.}
    \begin{tabular}{ccccccc}
    \toprule
    Rank  & Motif & ACC & & Rank  & Motif & ACC\\
    \midrule
    1     & M3    & 0.9809 & & 8  & M1    & 0.9686\\
    2     & M13   & 0.9802 & & 9  & M8    & 0.9477\\
    3     & M7    & 0.9791 & & 10  & M10    & 0.9408\\
    4     & M2    & 0.9789 & & 11  & M12    & 0.9377\\
    5     & M11   & 0.9781 & & 12  & M9    & 0.9317\\
    6     & M5    & 0.9762 & & 13  & M6    & 0.9271\\
    7     & M4    & 0.9717 & & \verb|-|  & ALL    & 0.9870\\
    \bottomrule
    \end{tabular}%
  \label{tab:important_motif}%
\end{table}%

As shown in Table \ref{tab:important_motif}, we can draw two conclusions. First, for a Chem2Bio2RDF network, the importance of different motifs varies. This is because important motifs often serve as building blocks within a network, and can even be used to define universal classes of the network they are in \cite{milo2002network}. For example, $M_{13}$ is a building block of the protein-protein interaction (PPI) network on the Chem2Bio2RDF graph (protein$\leftrightarrow$protein$\leftrightarrow$protein), and the PPI network is key for protein-disease association prediction, $M_{13}$ is thus ranked as one of top 3 most significant motifs on this dataset as shown in Table \ref{tab:important_motif}. Another example is the triangular motifs ($M_1$-$M_7$), which are essential in social networks due to their triadic closure nature \cite{benson2016higher, granovetter1973strength}. Second, the performance of the top three motifs is similar to the performance of all motifs combined (last row on the right). This is because a single motif may effectively encapsulate all of the network's essential information. With these two conclusions, we can see that one of the advantages of MGNN is its generality. That is, even if the importance of motifs is unknown, we can still use MGNN with all the motifs to achieve a final performance similar to that of using important motifs only. 

To further demonstrate the adaptive selection results of our motif redundancy minimization operator, we conduct the following experiments on Chem2Bio2RDF. We randomly sample 15 nodes and presented the representations of the top 3 and bottom 3 most significant motifs in Table \ref{tab:important_motif} by heatmap. As shown in Fig.~\ref{fig:heatmap}, representations w.r.t. unimportant motifs ($M_9$, $M_6$) are more sparse than other motifs. In addition, we observed that there are often only no more than three non-zero dimensions for a motif, which shows that MGNN actually needs a very low dimension to capture high-order structures.



\begin{figure}[b]
	\centering
	\includegraphics[width=\columnwidth]{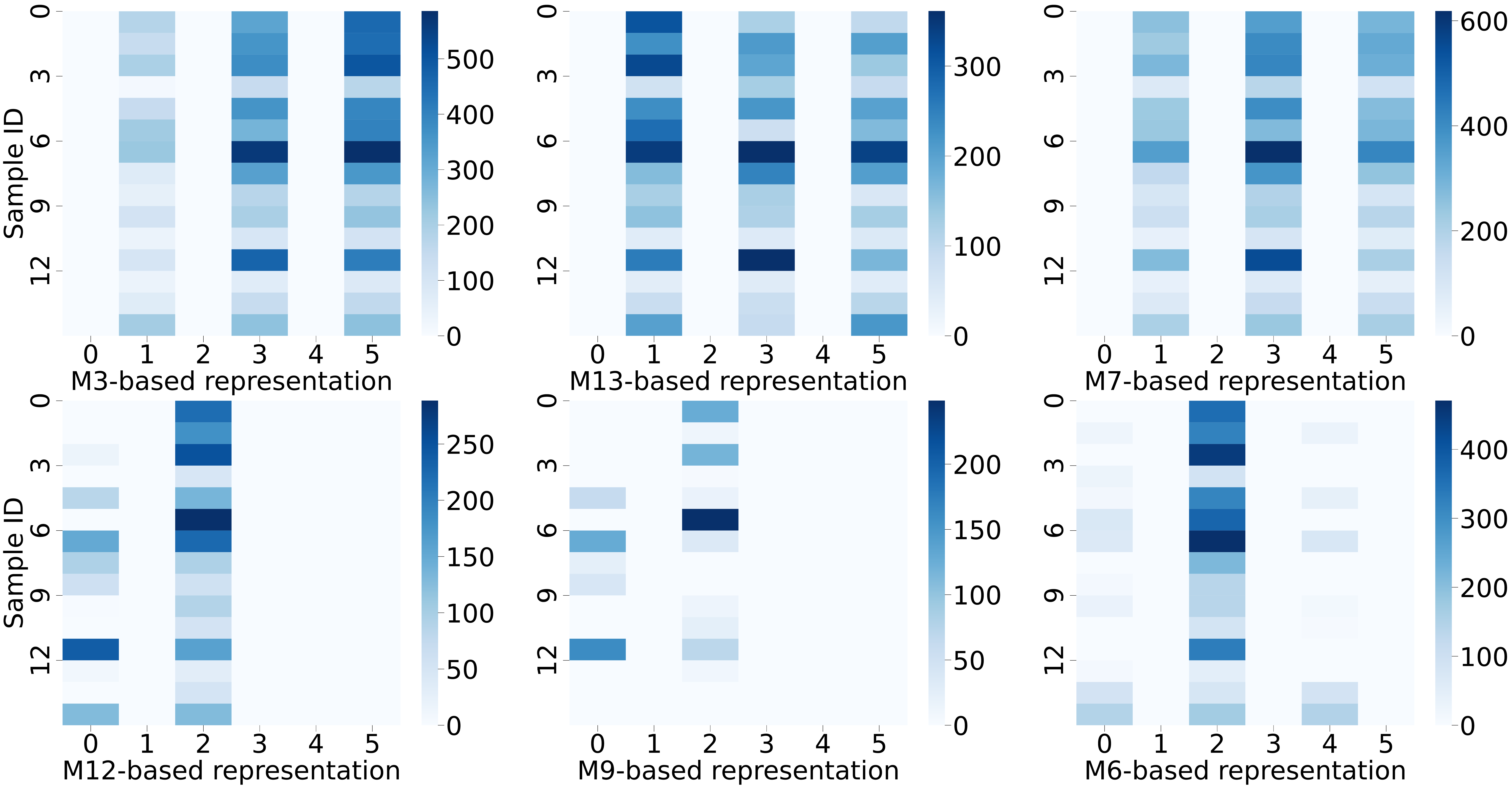} 
	\caption{Heatmap of the top 3 ($M_3, M_{13}, M_7$) and bottom 3 ($M_{12}, M_9, M_6$) most important motif-based representations on Chem2Bio2RDF.}
	\label{fig:heatmap}
\end{figure}


\begin{figure}[t!]
	\centering
	\includegraphics[width=0.65\columnwidth]{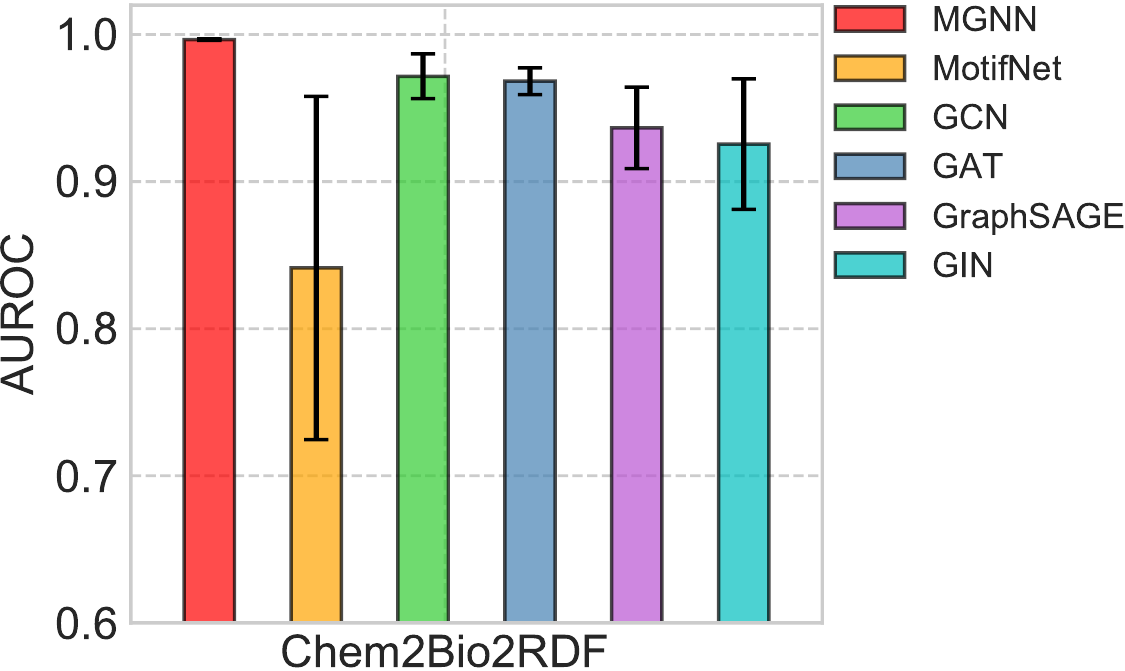} 
	\caption{Performance on protein-disease association prediction in Chem2Bio2RDF dataset, measured in AUROC. Standard deviation errors are given.}
	\label{fig:link_pred}
\end{figure}

\begin{table}[htbp]
  \centering
  \caption{Top predicted protein-disease associations with literature support.}
    \begin{tabular}{ccccc}
    \toprule
    Rank & Gene  & Disease &  PubMed ID \\
    \midrule
    1     & COX2  & Colorectal Carcinoma & 26 159 723 \\
    6     & CTNNB1 & Colorectal Carcinoma & 24 947 187 \\
    7     & P2RX7 & Colorectal Carcinoma & 28 412 208 \\
    12    & SMAD3 & Colorectal Carcinoma & 	30 510 241 \\
    16    & HRH1  & Colorectal Carcinoma & 30 462 522 \\
    37    & ABCB1 & Colorectal Carcinoma & 28 302 530 \\
    44    & AKT1  & Malignant neoplasm of breast & 29 482 551 \\
    64    & TP53  & Malignant neoplasm of breast & 31 391 192 \\
    82    & EP300 & Colorectal Carcinoma & 23 759 652 \\
    92    & ADORA1 & Colorectal Carcinoma & 27 814 614 \\
    107   & REN   & Renal Tubular Dysgenesis & 21 903 317 \\
    141   & FGFR2 & Autosomal Dominant & 16 141 466 \\
    157   & BCL2  & Non-Hodgkin Lymphoma & 29 666 304 \\
    169   & NOS2  & Malignant neoplasm of breast & 20 978 357 \\
    263   & CTNNB1 & Mental retardation & 24 614 104 \\
    \bottomrule
    \end{tabular}%
  \label{tab:ranking}%
\end{table}%

\subsubsection{The necessity of incorporating high-order structure information for prediction} We take protein-disease association prediction on the Chem2Bio2RDF dataset as an example to demonstrate the necessity of incorporating higher-order information from two aspects, i.e., illustration of validity and illustration of practicality. For an illustration of validity, we train MGNN and compare it to the baseline methods for protein-disease association prediction in terms of the area under the ROC curve (AUROC). For an illustration of practicality, we rank all the protein-disease pairs based on their predicted scores and then identify top pairs supported by existing publications.

MGNN is first trained to predict each protein-disease pair's associated score and compare it to the baseline approaches. Specifically, we will describe this step in detail by stating the background of the task as well as the specific experimental setup. As for the background of the task, protein-disease association prediction is a significant issue with the potential to give clinically actionable insights for disease diagnosis, prognosis, and treatment \cite{agrawal2018large}. The issue can be defined as predicting which proteins are associated with a given disease. Experimental methods and computational methods are the two primary kinds of current attempts to solve this challenge. Experimental methods for gene–disease association, such as genome-wide association studies (GWAS), and RNA interference (RNAi) screens, are costly and time-consuming to conduct. Therefore, a variety of computational methods have been developed to discover or predict gene–disease associations, including text mining, network-based methods \cite{ata2021recent}, and so on. Among them, network-based methods often need to use the structure information of the PPI network (constructed by $M_{13}$ motif). However, high-order PPI network structure is largely ignored in protein-disease discovery nowadays \cite{agrawal2018large}. Our MGNN thus can overcome this limitation.

We next describe the experimental setup. We mapped a protein to the gene that it is produced by, and viewed protein-disease association prediction as a link prediction task on the graph \cite{agrawal2018large}. We split the edges of the Chem2Bio2RDF dataset with the ratio of 85\%/5\%/10\% for training, validation and testing respectively. We adopted an inner product decoder for link prediction. The parameters of the model were optimized using negative sampling and cross-entropy loss and we used AUROC as a metric.
The number of the epoch was set to 1000 and the other hyperparameter settings are consistent with the node/graph classification task. Note that, the Chem2Bio2RDF dataset is missing a semantic mapping to disease IDs. To alleviate this problem, we search for genes associated with the disease (2929 known gene(protein)-disease links in Chem2Bio2RDF), and perform gene-disease association queries in the public database DisGeNET\footnote{https://www.disgenet.org/}, so as to realize the inference of the actual semantics of the disease IDs. Fig.~\ref{fig:link_pred} compares the performance of MGNN and other GNN-based methods under five random seeds. As can be seen, MGNN exceeds all GNN baselines on the AUROC metric. Importantly, the AUROC of MGNN is close to 100\%, and the standard deviation is very small, which means that MGNN has strong practicability in protein-disease association prediction.

For an illustration of practicality, we further ranked the whole unknown protein-disease pairs (over 28 million unknown pairs) based on their predicted scores, and identified 103 out of the top 1000 pairs that are supported by existing publications. Table \ref{tab:ranking} displays the first 15 of these 103 pairs, and the last column provides the PubMed ID of the publications that support our prediction.

As shown in Table \ref{tab:ranking}, all pairs have been validated by wet-labs and can be found in the DisGeNET database, e.g. row 2 (CTNNB1, Colorectal Carcinoma) is validated by RNAi screening \cite{tiong2014csnk1e}, and row 4 (SMAD3,Colorectal Carcinoma) is validated by GWAS \cite{huyghe2019discovery}. These methods predict protein-disease association from the angles which are orthogonal from MGNN. Therefore, we consider that they provide reasonable supports for our prediction. Taking the first protein-disease pair as an example, COX2 and Colorectal Carcinoma are reported in \cite{ahmed2015co} (i.e. PubMed ID: 26159723). In fact, COX2 is preferentially expressed in cancer cells and its expression is enhanced by proinflammatory cytokines and carcinogens \cite{ahmed2015co}. It is thus reasonable to predict a protein-disease association between them because there is evidence that the over-expression of COX2 is related to the infiltrating growth of Colorectal Carcinoma and other pathological characteristics \cite{tsunozaki2002cyclooxygenase}.

\subsection{Parameter Sensitivity}
\begin{figure}[t!]
	\centering
	\includegraphics[width=\columnwidth]{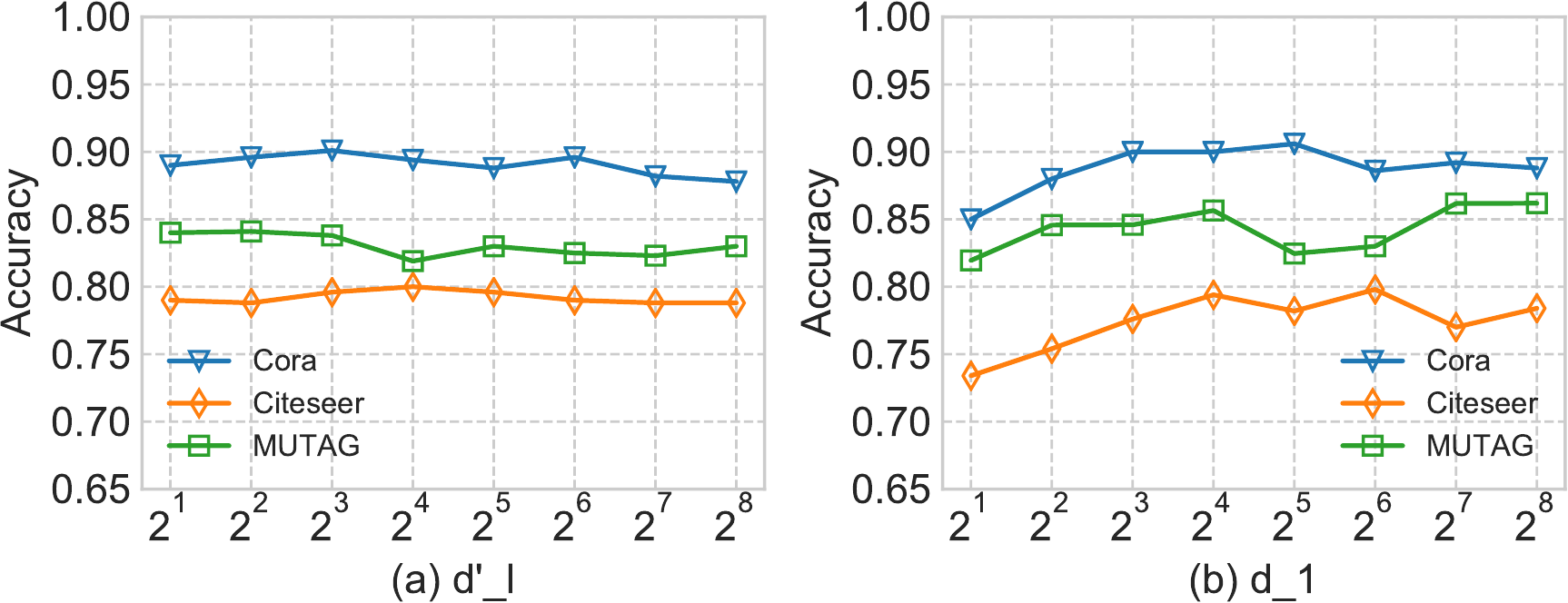} 
	\caption{Parameter sensitivity analysis for MGNN: (a) dimensionality $d^\prime_l$ in Eqs.~\eqref{equ:f1}--\eqref{equ:f2}.  (b) output dimensionality $d_1$ in Eq.~\eqref{equ:gcn}.}\label{fig:dim_sensitivity}
	
\end{figure}

We present the sensitivity analysis for the dimensionality parameters $d^\prime_l$ and $d_1$ in our MGNN.

In Fig.~\ref{fig:dim_sensitivity}(a), the performance of MGNN is not sensitive to changes in the dimensionality $d_l^\prime$ in Eqs.~\eqref{equ:f1}--\eqref{equ:f2}. Particularly, values of $d_l^\prime$ in the range $[2^2, 2^5]$ typically give a robust and reasonably good performance, e.g., $d^\prime_l = 6$ is a desirable choice in most cases. 
For the output dimensionality $d_1$ in Eq.~\eqref{equ:gcn},  as shown in Fig.~\ref{fig:dim_sensitivity}(b), the performance gradually improves and becomes stable around $2^4$, which is the preferred choice in most cases.

\subsection{Model Size and Efficiency}

\begin{table}[htbp]
  \centering
  \addtolength{\tabcolsep}{-1mm}
  \caption{Model size and efficiency analysis on the Pubmed dataset.}\label{tab:complexity}
    \begin{tabular}{lrrrrr}
    \toprule
     & \multirow{2}{*}{\# Params} & \multicolumn{2}{c}{Training} & Inference & \multirow{2}{*}{Accuracy} \\
     &  &  per epoch/s & overall/min   & /ms &  \\
    \midrule
    GAT        &   104,624    &    \textbf{0.01}   & \textbf{1.13} & \underline{4.00} & 0.8840 \\
    MixHop    &   \textbf{24,144}    &    \underline{0.02}  & 15.01  & 19.48 & 0.8628\\\midrule
    BGNN        &   9,866,596    &    1.97   & 608.64 & 4.57 & 0.8380\\
    EGAT       &   108,107    &  3.72     & 320.29 & 20.63 & 0.8970\\
    ESAGE       &   212,693    &  3.11     & 14.01 & 5.77 & \underline{0.9040}\\
    EGAT+SAGE       &   164,443    &  3.27     &  150.07 & 12.34 & 0.8992\\
    \midrule
    MGNN        &   \underline{26,084}    &    0.04   & \underline{10.00} & \textbf{1.25} & \textbf{0.9232}\\
    \bottomrule
    \end{tabular}%
  
\end{table}%


We evaluate the model size and efficiency of MGNN, in terms of the number of trainable parameters, training time (per epoch and overall), and inference time.
We select a representative baseline from standard GNNs (i.e., GAT) and high-order GNNs (i.e., MixHop), respectively, for comparison to MGNN. Moreover, since MGNN can be viewed as a model that integrates several motif-based modules, we also compare an ensemble GNN here (i.e., BGNN).  

For a more comprehensive comparison, we also develop a simple ensemble framework over 13 GNN modules (corresponding to our 13 motifs). First, we separately apply thirteen GNN modules that employ different initializations but otherwise the same input, and fuse their output by a fully connected layer. All hidden dimensions are set to 16. For the above framework, we develop three variants, respectively, the modules use only 13 GATs, only 13 GraphSAGEs as well as 7 GraphSAGEs and 6 GATs, denoted as EGAT and ESAGE, and EGAT+SAGE, respectively. 


As shown in Table \ref{tab:complexity}, 
MGNN is competitive in terms of model size and efficiency, while achieving the best accuracy. 
In particular, although several ensemble methods including EGAT, ESAGE and EGAT+SAGE achieve better accuracies among the baselines, their model sizes or efficiency are all worse than MGNN.
Note that the per epoch and overall training times are often inconsistent across methods, as a method may train faster per epoch but it converges slower, or vice versa. 
Further experiments involving ensemble GNNs on all datasets are presented in Section II of our supplementary materials.

\section{Conclusion}
We propose Motif Graph Neural Networks, a novel  framework to better capture high-order structures. Different from previous work, we propose the motif redundancy minimization operator and injective motif combination to improve the discriminative power of GNNs on the high-order structure. We also propose an efficient manner to construct a motif-based adjacency matrix. Further, we theoretically show that MGNN is provably more expressive than standard GNN, and standard GNN is in fact a special case of MGNN. Finally, we demonstrate that MGNN outperforms all baselines on seven public benchmarks.

\bibliographystyle{IEEEtran}
\bibliography{main}

\appendix

\begin{table*}[htbp]
  \centering
  \caption{Performance comparison on the node classification task, measured in accuracy. Standard deviation errors are given.}
  \label{tab:ensemble_nc}
    \begin{tabular}{lcccl}
    \toprule
           & Cora  & Citeseer  & Pubmed  & \multicolumn{1}{c}{Chem2Bio2RDF} \\
    \midrule
    GCN         & 0.8595 $\pm$ 0.0207 &   0.7764 $\pm$ 0.0045   & 0.8865 $\pm$ 0.0048      & 0.9371 $\pm$ 0.0017 \\
    GraphSAGE   &   0.8610 $\pm$ 0.0101    &  0.7744 $\pm$ 0.0061    &  0.8980 $\pm$ 0.0049     & 0.9630 $\pm$ 0.0010\\
    GAT        & 0.8775 $\pm$ 0.0127 &   \underline{0.7852} $\pm$ 0.0052  &  0.8840 $\pm$ 0.0079     & 0.9628 $\pm$ 0.0017\\
    GIN        &    0.8107 $\pm$ 0.0188  &   0.7255 $\pm$ 0.0160   &  0.8810 $\pm$ 0.0156      & 0.9205 $\pm$ 0.0129\\
    \midrule
    BGNN  & 0.8470 $\pm$ 0.0143 & 0.7750 $\pm$ 0.0112 & 0.8380 $\pm$ 0.0119 & 0.8746 $\pm$ 0.0115 \\
    EGAT  & 0.8720 $\pm$ 0.0040 & 0.7220 $\pm$ 0.0060 & 0.8970 $\pm$ 0.0010 &
    0.9658 $\pm$ 0.0040 \\
    ESAGE  & 0.8612 $\pm$ 0.0135 & 0.7604 $\pm$ 0.0171 & \underline{0.9040} $\pm$ 0.0102 & 0.9633 $\pm$ 0.0019\\
    EGAT+SAGE  & \underline{0.8792} $\pm$ 0.0102 & 0.7632 $\pm$ 0.0141 & 0.8992 $\pm$ 0.0109 & \underline{0.9663} $\pm$ 0.0010\\
    \midrule
    MGNN   & \textbf{0.9060} $\pm$ 0.0049 & \textbf{0.7948} $\pm$ 0.0050 & \textbf{0.9232} $\pm$ 0.0084 & \textbf{0.9870} $\pm$ 0.0021 \\
    \bottomrule
    \end{tabular}%
  
\end{table*}%

\subsection{Proof for Lemma 1}\label{sec:proof}
\begin{proof}
First, we show the relationship between the graph's adjacency matrix and the motif-based adjacency matrix. Then, using this relationship, we finish the proof of the lemma.

On the directed graph $G$ with self-loops, the subgraph composed of any node linked to any two of its neighbors is always an instance of open motif ($M_8$--$M_{13}$). That is, in the adjacency matrix $\mathbf{A}$ of the graph with self-loops, if $(\mathbf{A})_{ij} > 0$, $(\mathbf{A})_{uv} > 0$ where $(i,j)$ and $(u,v)$ are adjacent edges in $G$, then there always exist $k' \in \{8, 9, ..., 13\}$ such that $\mathbf{A}_{k'}$ satisfies $(\mathbf{A}_{k'})_{ij} > 0$, $(\mathbf{A}_{k'})_{uv} > 0$. 
It immediately follows that, on a graph with self-loops, if $(\mathbf{A})_{ij} > 0$, then we also have $(\mathbf{A}_{k'})_{ij} > 0$.
Without loss of generality, we assume $k'=13$ for ease of discussion later. That is, $\forall (i,j) \in \mathcal{E}$, $(\mathbf{A}_{13})_{ij} > 0$.




Next, we use the construction method to complete the proof of this lemma. Based on Table~I, an instance of standard GNN is in the form of $\tilde{\mathbf{h}}^{(l)}_v = \sigma (\omega \left(\left\{ (\mathbf{A})_{vi} \mathbf{W}_s^{(l)}\tilde{\mathbf{h}}^{(l-1)}_i \big| i \in \mathcal{N}(v) \right\}\right))$. 
We will use the following steps to find a special case of MGNN which have the same representational capacity as standard GNN. 

First, this special case of MGNN must satisfy the following equation. 
\begin{equation}\label{equ:raw_proof_target}
\begin{aligned}
&\big\|_{k=1}^{13} \sigma(\omega(\{ \alpha^{(l)}_{k,vi}\cdot(\mathbf{A}_k)_{vi} \mathbf{W}_m^{(l)}\mathbf{h}_i^{(l-1)} \big| i \in \mathcal{N}(v) \})) \\ 
&= \sigma(\omega( \{\big\|_{k=1}^{12} \mathbf{0}_k \big\|(\mathbf{A})_{vi}\mathbf{W}_s^{(l)}\tilde{\mathbf{h}}^{(l-1)}_i \big| i \in \mathcal{N}(v)\} )),
\end{aligned}
\end{equation}
where $\mathbf{0}_k$ is a $d_l$-dimensional zero vector, $\mathbf{W}_m^{(l)}$ and $\mathbf{W}_s^{(l)} \in \mathbb{R}^{d_l \times d_{l-1}}$,  $\mathbf{h}^{(l-1)}_i$ and $\tilde{\mathbf{h}}^{(l-1)}_i \in \mathbb{R}^{d_{l-1}}$, so that the dimensions on both sides of Eq.~\eqref{equ:raw_proof_target} are the same. That is, the output dimensions of the special case of MGNN
and standard GNN are the same, both being $13 d_l$. 

Next, with $\mathbf{W}_m^{(l)}$ and $\alpha^{(l)}_{k,vi}$ as variables, our goal is to prove that there will always be solutions to $\mathbf{W}_m^{(l)}$ and $\alpha^{(l)}_{k,vi}$ such that Eq.~\eqref{equ:raw_proof_target} holds.

For simplicity, in Eq.~\eqref{equ:raw_proof_target}, we use symbol $\varphi$, a aggregation function with activation, to represent $\sigma \circ \omega$, that is,
\begin{equation}\label{equ:proof_target}
\begin{aligned}
\big\|_{k=1}^{13} &\varphi(\{ \alpha^{(l)}_{k, vi}\cdot(\mathbf{A}_k)_{vi} \mathbf{W}_{m}^{(l)}\mathbf{h}^{(l-1)}_i  \big| i \in \mathcal{N}(v) \})\\ 
=& \varphi( \{\big\|_{k=1}^{12} \mathbf{0}_k \big\|(\mathbf{A})_{vi}\mathbf{W}_{s}^{(l)}\tilde{\mathbf{h}}^{(l-1)}_i \big| i \in \mathcal{N}(v)\} ).
\end{aligned}
\end{equation}


In the left hand side (LHS) of Eq.~\eqref{equ:proof_target}, the result will not change if the order of concatenation operation and aggregation $\varphi$ is exchanged. This is because the result value for each dimension in the LHS is only aggregated from the values of the same dimension in different feature vectors, and each feature vector is completely preserved after concatenation is performed. Thus, 
the LHS of Eq.~\eqref{equ:proof_target} becomes
\begin{equation}\label{equ:proof_target2}
 \varphi(\{ \big\|_{k=1}^{13} \alpha^{(l)}_{k,vi}\cdot(\mathbf{A}_k)_{vi}\mathbf{W}_{m}^{(l)}\mathbf{h}^{(l-1)}_i  \big| i \in \mathcal{N}(v)\}).
\end{equation}
By combining Eq.~\eqref{equ:proof_target}--\eqref{equ:proof_target2}, 
we get the equivalent form of Eq.~\eqref{equ:raw_proof_target}:
\begin{equation}\label{equ:proof_target4}
\begin{aligned}
&\varphi(\{  \big\|_{k=1}^{13} \alpha^{(l)}_{k, vi}\cdot(\mathbf{A}_k)_{vi}\mathbf{W}_{m}^{(l)}\mathbf{h}^{(l-1)}_i  \big| i \in \mathcal{N}(v)\}) \\ 
=& \varphi(\{\big\|_{k=1}^{12} \mathbf{0}_k \big\|(\mathbf{A})_{vi} \mathbf{W}_{s}^{(l)}\tilde{\mathbf{h}}^{(l-1)}_i \} \big| i \in \mathcal{N}(v)\}).
\end{aligned}
\end{equation}



Therefore, our goal now is to prove that there will always be solutions such that Eq.~\eqref{equ:proof_target4} holds. We can solve for the following Eqs.~\eqref{equ:final_target1}--\eqref{equ:final_target2} to ensure that Eq.~\eqref{equ:proof_target4} holds. For $k \in \{1, ..., 12\}$,
\begin{equation}\label{equ:final_target1}
\alpha^{(l)}_{k,vi} \cdot (\mathbf{A}_k)_{vi} \mathbf{W}_m^{(l)}\mathbf{h}^{(l-1)}_i = \mathbf{0}_k,
\end{equation}
and for $k=13$,
\begin{equation}\label{equ:final_target2}
\alpha^{(l)}_{13,vi} \cdot (\mathbf{A}_{13})_{vi} \mathbf{W}_m^{(l)}\mathbf{h}^{(l-1)}_i =(\mathbf{A})_{vi} \mathbf{W}^{(l)}_s  \tilde{\mathbf{h}}^{(l-1)}_i.
\end{equation}

Then we will demonstrate that $\forall l \ge 1$, there will always be solutions to $\mathbf{W}_m^{(l)}$ and $\alpha^{(l)}_{k,vi}$, such that Eqs.~\eqref{equ:final_target1}--\eqref{equ:final_target2} holds. Specifically, when $l=1$, $\mathbf{h}^{(0)}_i = \tilde{\mathbf{h}}^{(0)}_i = \mathbf{x}_i$, allowing Eqs.~\eqref{equ:final_target1}--\eqref{equ:final_target2} to hold for $\mathbf{W}_m^{(l)} = \frac{(\mathbf{A})_{vi}}{\alpha^{(l)}_{13,vi} \cdot (\mathbf{A}_{13})_{vi}} \mathbf{W}_s^{(l)}$, $\alpha^{(l)}_{13,vi} \ne 0$ and $\alpha^{(l)}_{k,vi} = 0$ ($k \in \{1, ..., 12\}$), that is, $1$-th special case of MGNN layer can generate the same vector representation  as $1$-th standard GNN layer since both models have the same output in the previous layer (i.e., $\mathbf{h}^{(0)}_i = \tilde{\mathbf{h}}^{(0)}_i$). Similarly, when $l > 1$, Eqs.~\eqref{equ:final_target1}--\eqref{equ:final_target2} holds. This finishes the proof of the lemma.
\end{proof}

\subsection{Performance Evaluation of Ensemble GNNs}\label{sec:ensemble}

\begin{table}[htbp]
  \centering
  \caption{Performance comparison on the graph classification task, measured in accuracy. Standard deviation errors are given.
  \label{tab:ensemble_gc}}
    \begin{tabular}{lccc}
    \toprule
           & MUTAG & ENZYMES & AIDS \\
    \midrule
    GCN          &    0.7555 $\pm$ 0.0651  &   0.2100 $\pm$ 0.0285     &  \underline{0.9895} $\pm$ 0.0091 \\
    GAT            &   0.7391 $\pm$ 0.0315  &   0.1667 $\pm$ 0.0000     &  0.8740 $\pm$ 0.1013 \\
    GraphSAGE    &    \underline{0.7984} $\pm$ 0.0526 &   0.2333 $\pm$ 0.0586     & 0.9855 $\pm$ 0.0091 \\
    GIN          & 0.7780 $\pm$ 0.0940    &   0.2630 $\pm$ 0.0330          & 0.9870 $\pm$ 0.0090\\
    \midrule
    EGAT   & 0.7820 $\pm$ 0.0610 & 0.2420 $\pm$ 0.0450 & 0.9850 $\pm$ 0.0050 \\
    ESAGE   & 0.7350 $\pm$ 0.0650 & \underline{0.2670} $\pm$ 0.0560 & 0.9850 $\pm$ 0.0060 \\
    EGAT+SAGE  & 0.7340 $\pm$ 0.0320 & 0.2500 $\pm$ 0.0480 & 0.9840 $\pm$ 0.0070 \\
    \midrule
    MGNN  & \textbf{0.8460} $\pm$ 0.0230 & \textbf{0.3070} $\pm$ 0.0300 & \textbf{0.9970} $\pm$ 0.0030 \\
    \bottomrule
    \end{tabular}%
\end{table}%

We evaluate the empirical performance of MGNN against
ensemble GNNs and standard GNNs in Table~\ref{tab:ensemble_nc} and Table~\ref{tab:ensemble_gc}.

As shown in Table~\ref{tab:ensemble_nc}, MGNN significantly and consistently outperforms all the baselines on different datasets.  In particular, ESAGE achieves the second best performance on Pubmed, while EGAT+SAGE achieves the second best performance on Cora and Chem2Bio2RDF. On Citeseer, GAT achieves the second best performance. MGNN is able to achieve further improvements against ESAGE by 2.12\% on Pubmed, against GAT by 1.22\% on Citeseer, as well as against EGAT+SAGE by 3.05\% and 2.14\% on Cora and Chem2Bio2RDF respectively.

In Table \ref{tab:ensemble_gc}, similarly,  MGNN regularly surpasses all baselines. In particular, ESAGE achieves the second best performance on ENZYMES, while GraphSAGE achieves the second best performance on MUTAG and GCN achieves the second best performance on AIDS.  Our MGNN is capable of achieving further improvements against ESAGE by 14.98\% on ENZYMES,  as well as against GraphSAGE and GCN by 5.96\% on MUTAG and by 0.76\% on AIDS, respectively.


\subsection{Efficiency Analysis of Motif-based Adjacency Matrix Construction}\label{sec:enumerate}

\begin{table}[htbp]
  \centering
  \chen{\caption{The efficiency analysis of three methods for constructing motif-based adjacency matrix, in terms of the running time (seconds). `MatMul' denotes matrix multiplication method.}\label{tab:enumerate}}%
    \begin{tabular}{lrcrrr}
    \toprule
          & \multicolumn{1}{c}{\multirow{2}[4]{*}{\# Nodes}} & \multicolumn{1}{c}{Closed Motif: M1} &       & \multicolumn{2}{c}{Open Motif: M13} \\
\cmidrule{3-3}\cmidrule{5-6}          &       & \multicolumn{1}{c}{MatMul [10]} &       & \multicolumn{1}{c}{Enumerate} & \multicolumn{1}{c}{\makecell{Non-\\enumerate}} \\
    \midrule
    Cora  & 2,708 & 0.003  &       & 73.322  & 1.534  \\
    Pubmed & 19,717 & 0.027  &       & 4249.435  & 18.852  \\
    \makecell{Chem2-\\Bio2RDF} & 295,911 & 0.228  &       & 1226K & 69.353  \\
    \bottomrule
    \end{tabular}%
\end{table}%

\chen{We evaluate the efficiency of MatMul \cite{zhao2018ranking} for closed motifs and our proposed non-enumeration method for open motifs, in terms of the running time, in Table \ref{tab:enumerate} below. For open motifs, we would compare the running time of both enumeration and non-enumeration methods.}

\chen{As shown in Table \ref{tab:enumerate}, it can be observed that MatMul can run very fast for closed motifs even for large-scale graphs, such as Chem2Bio2RDF. Meanwhile, compared to the standard enumeration method, our proposed non-enumeration method performs much better for open motifs. Even for Chem2Bio2RDF dataset, our non-enumeration can still run quite fast, taking about 69 seconds to construct the adjacency matrix for the open motif $M_{13}$. These results demonstrate that our preprocessing for both closed and open motifs is efficient.}



\subsection{Performance and efficiency analysis of MGNN using all motifs}

\begin{table}[htbp]
  \centering
  \chen{\caption{Performance and efficiency analysis of MGNN using all motifs or not, measured in accuracy and overall training time (minutes). `(M7, M8, M9)' denotes that MGNN uses only $M_7$, $M_8$ and $M_9$ motifs.}\label{tab:part_motif}}
    \begin{tabular}{lrrrrrr}
    \toprule
          &       & \multicolumn{2}{c}{ACC} &       & \multicolumn{2}{c}{Overall/min} \\
\cmidrule{3-4}\cmidrule{6-7}          & \multicolumn{1}{l}{\# Nodes} & (M7, M8, M9) & ALL   &       & (M7, M8, M9) & ALL \\
    \midrule
    Cora  & 2,708 & 0.8732 & 0.9060  &       & 0.87  & 1.37 \\
    CiteSeer & 3,327 & 0.7224 & 0.7948 &       & 0.80  & 1.29 \\
    PubMed & 19,717 & 0.4220  & 0.9232 &       & 5.73  & 10.00 \\
    \makecell{Chem2-\\Bio2RDF}   & 295,911 & 0.9741 & 0.9870  &       & 14.26 & 27.26 \\
    \bottomrule
    \end{tabular}%
\end{table}%

\chen{We compare the performance and efficiency of MGNN using all motifs or not, in terms of accuracy and overall training time. Specifically, we select motifs $M_7$, $M_8$ and $M_9$ which are commonly important in Cora, CiteSeer, PubMed and Chem2Bio2RDF, and make MGNN utilize just the above three motifs to conduct node classification on the four datasets. For simplicity, we denote this variant of MGNN as (M7, M8, M9).}

\chen{As shown in Table \ref{tab:part_motif}, MGNN using all the motifs achieves better accuracy, while (M7, M8, M9) method can clearly save the training time. However, (M7, M8, M9) method achieves lower accuracy than MGNN using all the motifs on all four datasets, showing that these three motifs are not sufficient to capture all the important high-order structures for these four datasets. In addition, we would think the efficiency when using all the motifs is still satisfactory. Even on the largest dataset (i.e., Chem2Bio2RDF), the overall training time for MGNN using all the motifs is just 13 minutes longer than (M7, M8, M9) method, while on other datasets the differences are much smaller.}

\end{document}